\RequirePackage{fix-cm}
\documentclass[smallextended]{svjour3}       
\smartqed  
\usepackage{graphicx}
\usepackage{amssymb,amsmath,amsfonts}
\usepackage[shortlabels]{enumitem}
\usepackage{float}
\usepackage{cite}
\newtheorem{algorithm}{Algorithm}
\newtheorem{assumption}{Assumption}

\newcommand*{\ee}{{\rm e}}
\newcommand*{\jj}{{\rm j}}

\newcommand{\vecto}{{\rm vect}}
\newcommand{\conc}{{\rm conc}}
\newcommand{\dom}{{\rm dom}}
\newcommand{\intr}{{\rm int}}

\newcommand{\paren}[1]{\left(#1\right)}

\newcommand{\norm}[1]{\left\| #1 \right\|}
\newcommand{\Ebb}{\mathbb{E}}

\begin{document}

\title{Nonuniform Defocus Removal for Image Classification
	\thanks{This project has received funding from the ECSEL Joint Undertaking (JU) under grant agreement No. 826589. The JU receives support from the European Union's Horizon 2020 research and innovation programme and Netherlands, Belgium, Germany, France, Italy, Austria, Hungary, Romania, Sweden and Israel.}
}
\titlerunning{Nonuniform Defocus Removal for Image Classification}

\author{Nguyen Hieu Thao, Oleg Soloviev, Jacques Noom and Michel Verhaegen}


\institute{Nguyen Hieu Thao\at
	Delft Center for Systems and Control,
	Delft University of Technology, 2628CD Delft, The Netherlands. Department of Mathematics, School of Education, Can Tho University, Can Tho, Vietnam.
	\email{h.t.nguyen-3@tudelft.nl, nhthao@ctu.edu.vn}\\
	Oleg Soloviev\at
	Delft Center for Systems and Control,
	Delft University of Technology,
	2628CD Delft, The Netherlands.
	Flexible Optical B.V., Polakweg 10-11, 2288 GG Rijswijk, The Netherlands.
	\email{o.a.soloviev@tudelft.nl}\\
	Jacques Noom\at
	Delft Center for Systems and Control,
	Delft University of Technology,
	2628CD Delft, The Netherlands.
	\email{j.noom@tudelft.nl}\\
	Michel Verhaegen\at
	Delft Center for Systems and Control,
	Delft University of Technology,
	2628CD Delft, The Netherlands.
	\email{m.verhaegen@tudelft.nl}
}


\maketitle

\begin{abstract}
	
	We propose and study the single-frame anisoplanatic deconvolution problem associated with image classification using machine learning algorithms, named the \textit{nonuniform defocus removal} (NDR) problem.
	Mathematical analysis of the NDR problem is done and the so-called \textit{defocus removal} (DR) algorithm for solving it is proposed. Global convergence of the DR algorithm is established without imposing any unverifiable assumption.
	Numerical results on simulation data show significant features of DR including solvability, noise robustness, convergence, model insensitivity and computational efficiency. Physical relevance of the NDR problem and practicability of the DR algorithm are tested on experimental data. Back to the application that originally motivated the investigation of the NDR problem, we show that the DR algorithm can improve the accuracy of classifying distorted images using convolutional neural networks.
	The key difference of this paper compared to most existing works on single-frame anisoplanatic deconvolution is that 
	the new method does not require the data image to be decomposable into isoplanatic subregions.
	Therefore, solution approaches partitioning the image into isoplanatic zones are not applicable to the NDR problem and those handling the entire image such as the DR algorithm need to be developed and analysed.
	
	\keywords{Anisoplanatic deconvolution, Nonuniform defocus removal, Image reconstruction, Computational imaging, Inverse problems, Optimization algorithm, Numerical analysis, Image classification}
	
\end{abstract}

\section{Introduction}\label{s:introduction}
Improving the quality of images degraded by optical aberrations or imperfections of photographic settings is an important and active research topic with applications in many fields of imaging science such as machine vision \cite{SonHlaBoy14, FenBou15}, medical imaging \cite{FosHun79}, astronomy \cite{Rod99}, microscopy \cite{BooNeiJusWil02, Ji17}, adaptive optics \cite{FliRig05}, \textit{etc}.
In the simplest case, the blurred image can be modelled as the convolution of the object with a single and known point spread function (PSF). This problem, known as \textit{isoplanatic deconvolution} or \emph{deblurring}, has been intensively studied, see, \textit{e.g.}, \cite{Can76, DaiFie87, WilSolPozVdoVer17}.
In the more challenging case where the distortion effects are spatially variant over the image, the problem known in astronomy and microscopy as \emph{anisoplanatic deconvolution} \cite{FliRig05, VorCar01, ChaWuTsa17, Pozzi20} and in machine vision and consumer imaging as \emph{nonuniform deblurring} \cite{Nagy1998, Bardsley2006, Kim2016, Pan2019} in essence requires the evaluation of the blurring operator at individual image pixels.

In many applications of nonuniform deblurring \cite{Ji17, Nagy1998, Bardsley2006, Kim2016, Pan2019, Kieu2016, SroKamLu16} the blurring profile is approximately decomposable into a finite number of isoplanatic zones (piecewise constant), and standard  deconvolution approaches can be carried out in individual zones and the solution is then synthesized.
However, this approach is not applicable to problems in which the data image cannot be separable for the isoplanatic zones of the object due to the large number of the required zones or their small size compared to the size of the PSFs.
Such blur is in the literature often referred to as \textit{pixelwise} as opposed to the \textit{piecewise constant}.
A number of solution approaches for pixelwise deblurring inherently tied to a particular application have been proposed and developed.
In astronomical imaging, the approach in which the missing blur kernels are linearly interpolated was proposed in \cite{Nagy1998} and improved later in \cite{Hirsh2010, Thiebaut2016}.
In camera shake removal, the approach in which the blur is parametrised and the parametrisation is then exploited in the deconvolution (\textit{e.g.}, through the trajectory of the camera pose) was used in \cite{Kim2016, Hirsch2011, Yue2015}.
In motion deblurring, the approach making use of the depth information of the scene given by an additional camera or estimated from the data image was considered in \cite{Pan2019, LiXu2012, Xu2016}.

Different analysis is often required for each specific application of nonuniform deblurring based on the structure of its own blurring profile and the prior knowledge of the solution.
In this paper we propose to study a practically relevant application of pixelwise deblurring in \textit{image classification}.
Having a well-trained neural network for classifying standard (undistorted) images at our disposal, we want to make it work also for nonuniformly distorted images.
The research question was motivated by the rapidly growing applications of computer vision in fault detection and feature recognition using machine learning algorithms.
The influence of various image degradation modalities (\textit{e.g.}, isoplanatic defocus, motion blur, noise corruption, low-resolution, fisheye lens, underwater effects, \textit{etc}) on the performance of the classification algorithm has been studied in the literature, see, \textit{e.g.}, \cite{Pei2019, Endo2020} and the references therein.
In this paper we analyse the distorted images caused by spatially varying defocus effects, particularly due to imperfections of the photographic settings.
To make machine learning algorithms able to work with such distorted images while being trained with undistorted ones, we propose to  perform an additional image preprocessing step to correct the distorted images before classification instead of modifying the neural network. The result will be an improvement of the accuracy of the classification algorithm.
Motivated by these meaningful applications, this paper is devoted to studying the \textit{single-frame anisoplanatic deconvolution} problem with data distorted by nonuniform defocus effects.

Blurring effects caused by nonuniform defocus occur in all imaging scenarios with fast aperture where the depth coordinate is not (approximately) constant over the object to be imaged.
The object depth profile determined by the object shape and position relative to the camera is often required for a complete anisoplanatic deconvolution.
However, the object shape is available only in problems with \textit{a priori} known reference, in particular, it is irrelevant in the analysis of this paper which is associated with the application in image classification.
We thus restrict the paper objective to removing nonuniform defocus effects caused by imperfections of the imaging angles.
More specifically, we consider two-dimensional pictorial objects degraded by nonuniform defocus induced by the inclination (non perpendicularity) of the camera axis with respect to the object plane, see Sect. \ref{subs:application model} for the problem formulation and Sect. \ref{s:image classification} for the interested application.
The underlying condition that the data image cannot be separated for isoplanatic subregions of the object is the main difference of this paper compared to most existing works on single-frame anisoplanatic deconvolution, see also Remark \ref{r:indecomposable data}.
This is roughly explained visually in Fig. \ref{fig:intro} where the first figure depicts an object depth profile for which the data image can be (approximately) separable for individual defocus zones while the second illustrates an inseparable case where the constant defocus zones are too small.
In other words, we analyse the case where nonuniform defocus appears in an individual object (\textit{e.g.}, a single coin, an alphabetical letter, a welding spot, \textit{etc}) in contrast to the case where multiple objects each with different but (approximately) constant defocus value appear in the image.
As a consequence, solution approaches partitioning the image into isoplanatic zones are not applicable to the problem considered in this paper and global algorithms handling the entire image need to be developed.

\begin{figure}[ht]
	\centering
	\includegraphics[width=.42\columnwidth]{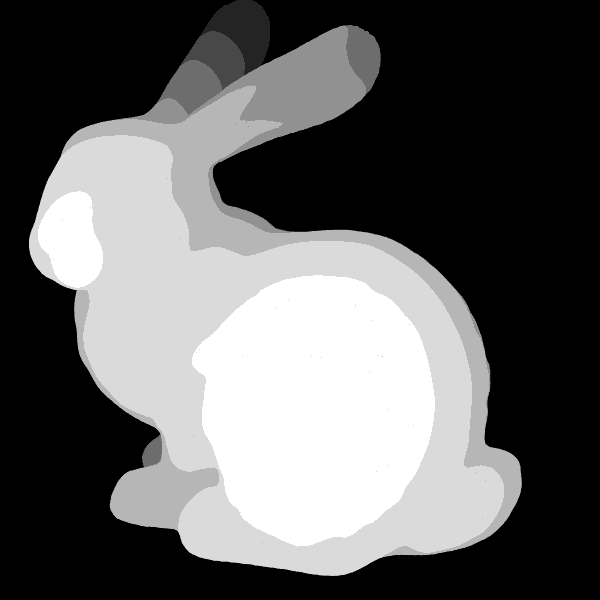}(a)\hspace{.05\columnwidth}%
	\includegraphics[width=.42\columnwidth]{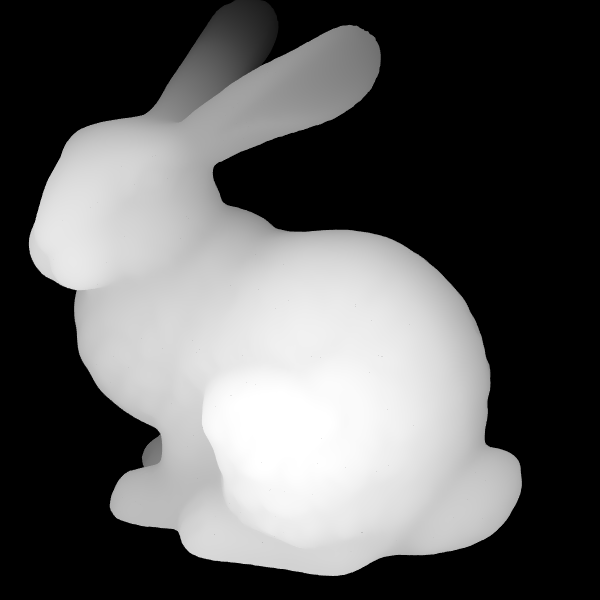}(b)
	\caption{Synthetic data depicts (a) an object depth profile for which the data image can be approximately separable for individual defocus zones of the object versus (b) an \textit{inseparable} case where the constant defocus zones are too small. Solution approaches partitioning the data into isoplanatic zones are applicable to (a) but not (b) for which global algorithms handling the entire data are required.}
	\label{fig:intro}
\end{figure}

The organization and the main contribution of this paper are as follows. The single-frame anisoplanatic deconvolution problem associated with image classification using machine learning algorithms is proposed in Sect. \ref{s:problem formulation}.
For brevity, it is called the \textit{Nonuniform Defocus Removal} (NDR) problem in the sequel.
In Sect. \ref{s:solution algorithm} we frame the NDR problem in an optimization context such that it enables the application of almost any state-of-the-art optimization algorithm.
Particularly, using the algorithmic scheme introduced by Beck and Teboulle \cite{BecTeb09} generally known as FISTA, we develop the so-called \textit{Defocus Removal} (DR) algorithm for the NDR problem.
Sparsity and parallel structure of the NDR problem is exploited to improve computational efficiency of DR compared to the direct implementation of FISTA, see the details in Algorithm \ref{al:fast gradient}.
The DR algorithm, which is the combination of FISTA and Algorithm \ref{al:fast gradient}, can be viewed as a computationally efficient version of FISTA applied to the NDR problem.
It is worth mentioning that comparison of DR with other state-of-the-art algorithms is not a goal of this paper as the efficiency of FISTA has been well known.
Convergence analysis of the DR algorithm and its conditions are presented in Sect. \ref{s:convergence analysis}.
In particular, we establish a global convergence criterion for DR without imposing any unverifiable mathematical assumption (Theorem \ref{t:NDR convergence}).
Numerical results on simulation data showing significant features of the DR algorithm including solvability, noise robustness, convergence, model insensitivity and computational efficiency are reported in Sect. \ref{s:numerical results}.
In Sect. \ref{s:exp_data} physical relevance of the NDR problem and practicability of the DR algorithm are tested on experimental data.
We demonstrate in Sect. \ref{s:image classification} how the DR algorithm can be used to improve the accuracy of classifying distorted images using convolutional neural networks.
This result shows that the proposed solution approach for nonuniform defocus removal indeed works for the target application that originally motivated the investigation of the NDR problem.
Concluding remarks are given in Sect. \ref{s:conclusion}.

\section{Problem formulation}\label{s:problem formulation}

In this paper we consider visually degraded images caused by spatially varying defocus, \textit{i.e.}, the object to be imaged does not lie in a single depth-of-field (defocus zone) of the camera.
Consider the object $o$ locating in a number of defocus zones denoted by $D_n$ ($n=1,2,\ldots,N$), then the isoplanatic section of $o$ in the $n$th zone is determined by the binary mask $\mu_n$ given by
\[
\mu_n(\xi,\eta) = \begin{cases}
	1 & \mbox{if } o(\xi,\eta) \in D_n,\\
	0 & \mbox{otherwise},	 
\end{cases}
\]
where and elsewhere in this paper $(\xi,\eta)$ denotes the coordinates of the referred two-dimensional array.
Both the defocus zones and the masks are determined by the \textit{depth profile} of the object which is assumed to be known in this paper.

\subsection{Nonuniform defocus blurring}\label{subs:NDB}

The Point Spread Function (PSF) for each defocus zone $D_n$ can be modelled using the Fourier transform \cite{Goo05}:
\begin{equation}\label{p_n}
	p_n = \left|\mathcal{F}\paren{A \odot \ee^{\jj d_n Z_2^0}}\right|^2\quad (n=1,2,\ldots,N),
\end{equation}
where $\mathcal{F}$ is the 2-dimensional Fourier transform, $A$ is the \textit{generalized pupil function}, $\odot$ is the elementwise product, $Z_2^0$ is the Zernike polynomial of order two and azimuthal frequency zero, $\jj=\sqrt{-1}$ is the imaginary unit, and $d_n$ is the \textit{normalized depth} of the defocus zone {$D_n$}.
For each $n=1,2,\ldots,N$, the normalized depth $d_n$ is linearly proportional to the depth (\textit{i.e.}, the $z$-coordinate) of the zone $D_n$ and the linear coefficient is determined from the camera specifications and the photographic settings.
In this paper the optical aberration of the camera is presumably ignorable, \textit{i.e.}, the diffraction-limited PSF is not resolved by the camera pixels and $A$ represents the known binary function representing the camera aperture.

Having partitioned the object $o$ into isoplanatic sections using the masks $\mu_n$, we can model its image as follows:
\begin{equation}\label{forward model}
	i = \sum_{n=1}^{N}\paren{\mu_n\odot o} * p_n\; +\; \omega,
\end{equation}
where $*$ is the 2-dimensional convolution operation, and $\omega$ represents the (unknown) discrepancy between the theoretically predicted data and the actually measured one due to noise, model deviations, \textit{etc}.

\begin{remark}[decomposable object versus indecomposable data]\label{r:indecomposable data}
	In this study the cause of the nonuniform defocus to be removed is due to the fact that the object does not lie in a single depth-of-field of the camera. It is worth keeping in mind that the object itself is \emph{decomposable} into isoplanatic defocus zones.
	However, the data image is assumed to be \emph{inseparable} for those zones.
	In other words, the image $i$ formulated in (\ref{forward model}) is not (approximately) separable for the individual terms $(\mu_n\odot o)*p_n$, for instance, as the zones $\mu_n\odot o$ are too small compared to the PSFs although they are mutually separated.
	The underlying assumption of \emph{indecomposable} data is the main difference of this paper compared to most existing works on single-frame nonuniform deblurring.
\end{remark}

\subsection{Inverse problem}

The inverse problem under investigation is to restore the object $o$ from the (single) image $i$ given by (\ref{forward model}). In this paper the error term $\omega$ in (\ref{forward model}) is assumed to be Gaussian.
It is worth mentioning that this assumption does not exclude the case of Poisson noise since in view of the central limit theorem, Poisson noise can be well approximated by a Gaussian distribution provided that the image is registered with a sufficiently large number of photon counts.
The \textit{maximum-likelihood} approach applied to equation (\ref{forward model}) leads to the following linear least squares problem:
\begin{equation}\label{OP1}
	\min\; f(o) := \frac{1}{2}\norm{i - \sum_{n=1}^{N}\paren{\mu_n\odot o} * p_n}^2,
\end{equation}
where $\|\cdot\|$ denotes the Frobenius norm. In spite of its simplicity with the only unknown $o$, this single-frame isoplanatic deconvolution problem is ill-posed, especially in the presence of noise and model deviations. 

In this paper the object $o$ is assumed to take values in $[0,1]$ and the set of all matrices satisfying this constraint is denoted by $\Omega$. It has been widely known that taking this constraint into account can improve the robustness of deconvolution algorithms (\textit{e.g.}, \cite{WilSolPozVdoVer17}). Among standard constrained optimization techniques, we choose to reformulate (\ref{OP1}) with the additional constraint $o\in \Omega$ as the following unconstrained optimization problem using the indicator function \cite{BoyVan04}:
\begin{equation}\label{NDRP}
	\min\; F(o) := f(o) + \iota_{\Omega}(o),
\end{equation}
where $\iota_{\Omega}(o)=0$ if $o\in \Omega$ and $\iota_\Omega(o)=\infty$ otherwise, and $f$ is given by (\ref{OP1}).
In the sequel (\ref{NDRP}) will be referred to as the \emph{Nonuniform Defocus Removal} (NDR) problem.

\subsection{Application model in image classification}\label{subs:application model}

The depth profile of the object plays a central role in studying the NDR problem. Different analysis is required for different application of (\ref{NDRP}) based on the particular structure of the depth profile of the object. Motivated by its practical application in computer vision using \textit{convolutional neural networks} (see Sect. \ref{s:image classification} for an example in image classification), in the sequel we focus on two-dimensional pictorial objects degraded by nonuniform defocus induced by the inclination (non perpendicularity) of the camera axis with respect to the object plane. It is worth mentioning that we are not trying to eliminate all sources of defocus in the image, especially those inherent from three-dimensional scenes or intrinsic to the image itself (for instance, being a blurry photo), but rather the one caused by imperfections of the photographic settings although the designed algorithm can also be extended to handle any known depth profile.

Up to a rotation the defocus zones of a (square) pictorial object can be assumed to be equal and horizontal.\footnote{We work with \textit{vertical} defocus zones in Sect. \ref{s:exp_data}.}
The number of pixel rows per zone is called the \textit{Depth of Field} (DoF) and denoted by $s$ in the sequel.
Let the $z$-coordinate be monotone decreasing along the object rows and the defocus zones be numbered also starting from the top. Then the normalized depth $d_n$ of defocus zone $D_n$ is linearly dependent on the ordering number $n$ with the linear coefficient denoted by $d$, that is,
\begin{equation}\label{d_n}
	d_n = d(n_0-n),\quad (n=1,2,\ldots,N),
\end{equation}
where $n_0$ is the position of the in-focal zone (without defocus).

The two parameters $d$ and $n_0$ are respectively termed as the \textit{blur coefficient} and the \emph{focal position}. They are given by the camera specifications and the tilt angle of the optical axis relative to the object plane. Both are assumed to be known in this paper except in Sect. \ref{subs:sensitivity} where sensitivity of the proposed algorithm with respect to them is analysed.
In the analysis with experimental data in Sect. \ref{s:exp_data}, we propose \textit{ad-hoc} numerical approaches for estimating them using \textit{guide-star} objects (see Sect. \ref{subs:exp_params}). For convenience, the technical parameters of the NDR problem (\ref{NDRP}) applied to image classification formulated in this section are listed in Table \ref{tbl:params}.

\begin{table}[ht]
	\caption{Parameters of the NDR problem (\ref{NDRP}) applied to image classification}
	\label{tbl:params}
	\centering{
		\begin{tabular}[1\baselineskip]{|c|c|c|c|}
			\hline
			$N$ & $s$ & $d$ & $n_0$\\ \hline
			\#DoF zones & DoF size & blur coefficient & focal position\\ \hline
		\end{tabular}
	}
\end{table}

We conclude this section by discussing major challenges in inverting (\ref{forward model}) as well as in solving (\ref{NDRP}). First, typical difficulties of single-frame deconvolution/deblurring are relevant for the more challenging problem under consideration with spatially varying blurring effect.
Second, the inaccuracy of measurements near the boundary causes sequential errors in restoring the object.
Finally, being the main difference from the literature of anisoplanatic deconvolution, the data image is not separable for individual blurring kernels as illustrated in Fig. \ref{fig:intro}.
Consequently, global solution approaches processing the entire data image need to be developed and analysed.
Nonseparability of the data for individual blurring kernels also refrains us from addressing the NDR problem in the frequency domain.

\section{Solution algorithms}\label{s:solution algorithm}

We first observe that (\ref{NDRP}) can be cast in the framework of convex optimization as both functions $f$ and $\iota_\Omega$ are convex. This suggests that the NDR problem can be addressed using state-of-the-art algorithms in convex optimization. Thus instead of trying creating new solution methods for it which would likely end up with reinventing the wheel, we make use of existing efficient algorithms (with necessary modifications) and rigorously explain why they should work.

\subsection{Proximal gradient methods}

The mathematical properties of $f$ and $\Omega$, shown later in Lemma \ref{l:smooth of f}, suggest that the composite optimization model (\ref{NDRP}) can be efficiently solved by \emph{proximal gradient methods}.
To proceed further in this direction, we need to calculate the metric projection on the constraint set $\Omega$ and the gradient of the cost function $f$.

\begin{lemma}[projection on $\Omega$]\label{l:P_Ocal}
	The projection on $\Omega$ is given by
	\begin{equation*}
		P_\Omega(x)(\xi,\eta) =
		\begin{cases}
			0 & \mbox{if }\;\; x(\xi,\eta) < 0,\\
			x(\xi,\eta) & \mbox{if }\;\;  0\le x(\xi,\eta) \le 1,\\
			1 & \mbox{if }\;\; x(\xi,\eta) > 1.
		\end{cases}
	\end{equation*}
\end{lemma}

\begin{lemma}[gradient of $f$]\label{l:nabla f}
	The function $f$ given by \eqref{OP1} is differentiable everywhere with the gradient given by
	\begin{equation}\label{nabla f}
		\nabla f(x) = -\sum_{n=1}^{N}\mu_n\odot \paren{p_n^{\dagger} * \paren{i - \sum_{l=1}^{N}\paren{\mu_l\odot x} * p_l}},
	\end{equation} 
	where $p_n^{\dagger}$ denotes the reflection of $p_n$ via its origin (centre),
	\begin{equation*}
		p_n^{\dagger}(\xi,\eta) := p_n(-\xi,-\eta)\quad (n=1,2,\ldots,N).
	\end{equation*}
\end{lemma}

Lemma \ref{l:P_Ocal} is standard while the proof of Lemma \ref{l:nabla f} is given in Appendix \ref{a:nabla f}.
In the sequel we will utilize outstanding features of the fast proximal gradient method introduced by Beck and Teboulle \cite{BecTeb09}, often known as FISTA. This famous algorithm achieves an $O(1/k^2)$ rate of convergence in objective values compared to $O(1/k)$ of the standard (without acceleration) proximal gradient method \cite[Chapter 10]{Bec17}. For completeness, let us recall FISTA with constant stepsize applied to (\ref{NDRP}).

\begin{algorithm}[the FISTA algorithm \cite{BecTeb09}]\label{al:FISTA}\ 
	
	\emph{Input:}
	$o^{(0)}$ --- the initial guess of $o$, $x^{(0)} = o^{(0)}$ --- initial auxiliary parameter, $\lambda>0$ --- the stepsize, $t^{(0)}$ --- the initial acceleration parameter, and $K$ --- the number of iterations.
	
	\emph{Iteration process}: for $k=1,2,\ldots,K$, compute the updates as follows:
	\begin{equation}\label{FISTA:update}
		\begin{gathered}
			x^{(k+1)} = P_\Omega\paren{o^{(k)} - \lambda \nabla f\paren{o^{(k)}}},
			\\
			t^{(k+1)} = \frac{1+\sqrt{1+4{t^{(k)}}^2}}{2},
			\\
			o^{(k+1)} = x^{(k+1)} + \frac{t^{(k)}-1}{t^{(k+1)}}\paren{x^{(k+1)}-x^{(k)}}.
		\end{gathered}
	\end{equation}
\end{algorithm}

\subsection{Evaluation of $\nabla f$}

Direct implementation of Algorithm \ref{al:FISTA} to the NDR problem, however, encounters a major obstacle in terms of computational complexity due to the highly expensive formula (\ref{nabla f}) of $\nabla f$.
This significantly limits the applicability of Algorithm \ref{al:FISTA} to practical problems.
Fortunately, to some extent this barrier can be overcome for the interested application in image classification formulated in Sect. \ref{subs:application model} where the uniform structure of the masks $\mu_n$ enables faster evaluation of $\nabla f$.

The simplicity of the masks $\mu_n$ allows us to exploit their sparsity in computing $\nabla f(x)$ numerically.
As a result, we propose a fast method for evaluating (\ref{nabla f}) which is many times faster than the naive computing of (\ref{nabla f}), see also Sect. \ref{subs:complexity}.
The following algorithms are presented in details for full-size convolution operation and $\rho$ odd.

\begin{algorithm}[fast evaluation of $\nabla f$]\label{al:fast gradient}\ \\
	\emph{Input:} $x$ --- array of size $(l,w)$,
	
	\hspace*{0.54cm}  $p_n$ --- PSFs of size $(\rho, \rho)$,
	
	\hspace*{0.54cm} $i$ --- image of size $(l+\rho-1, w+\rho-1)$.\\	
	\emph{Initials:}\; $r=i$,\; $R = {\rm zeros}(N,s+2\rho,w+2\rho)$,
	
	\hspace*{0.75cm} $G = {\rm zeros}(l,w)$.\\
	\emph{Main calculations:}
	\begin{enumerate}
		\item reshape $x$ to the $(N,s,w)$-tensor $X$
		\item stack $p_n$ to form the $(N,\rho,\rho)$-tensor $P$
		\item $I=X*P$ --- convolution over $N$ channels
		\item For $n=1:N$, $r[(n-1)s:ns+\rho-2,:]$\; +=\; $I[n,:,:]$
		\item For $n=1:N$,\\
		\hspace*{0.5cm} $R[n,:,:]$\; +=\; $r[(n-1)s:ns+\rho-2,:]*p_n^{\dagger}$\\
		\hspace*{0.5cm} $G[(n-1)s:ns-1,:]$\; +=\; $R[n,\rho:-\rho,\rho:-\rho]$
	\end{enumerate}
	\emph{Output:} $\nabla f(x) = G$.
\end{algorithm}

\subsection{The proposed algorithm}

By integrating Algorithm \ref{al:fast gradient} into FISTA, we can propose an efficient solution method for the NDR problem applied to image classification, which is of our main interest in this paper.
In the sequel it is referred to as the \textit{Defocus Removal} (DR) algorithm.
The relevant parameters have been explained in Table \ref{tbl:params} and earlier in this section while the repetitive term $(n=1,2,\ldots,N)$ following the subscript $n$ is dropped for brevity.
We remark that the initial guess $o^{(0)}$ used in Algorithm \ref{al:NDR} results in better restoration compared to, \textit{e.g.}, random initials, especially in terms of convergence speed.

\begin{algorithm}[Defocus Removal (DR) algorithm]\label{al:NDR}\ \\
	\emph{Input:} $N$, $s$, $d$, $n_0$, $\rho$, $l$, $w$, image $i$ of size $(l+\rho-1, w+\rho-1)$, $\lambda>0$, $t^{(0)}$, and $K$.\\
	\emph{Initializations:}
	\begin{enumerate}[a)]
		\item compute $p_n$ according to (\ref{p_n}) with $d_n$ given by (\ref{d_n}),
		\item\label{initial_guess} set $o^{(0)}$ equal the $(l,w)$-central part of $i$,
		\item set $x^{(0)} = o^{(0)}$.	
	\end{enumerate}
	\emph{Iteration process}: for $k=1:K$,
	\begin{enumerate}
		\item evaluate $\nabla f\paren{o^{(k)}}$ using Algorithm \ref{al:fast gradient},
		\item update $x^{(k+1)}$, $t^{(k+1)}$ and $o^{(k+1)}$ according to (\ref{FISTA:update}).
	\end{enumerate}
	\emph{Output:} reshape $o^{(K)}$ to an array of size $(l,w)$.
\end{algorithm}

\begin{remark}[parallelizability of the DR algorithm]
	In Algorithm \ref{al:fast gradient} the array $x$ of size $(l,w)$ is reshaped to a 3-order tensor of size $(N,s,w)$ where $l=Ns$ and the blurring kernels $p_n$ of size $(\rho,\rho)$ are treated as a 3-order tensor of size $(N,\rho,\rho)$.
	These tensors in turn enable parallel computing of the convolution operations in (\ref{nabla f}) as explained in Appendix \ref{a:parallel computing}.
	This particularly shows that Algorithm \ref{al:fast gradient} and thus the DR algorithm are parallelizable.
	Such parallel computing would unfortunately require additional \textit{GPU hardwares} and has not been numerically implemented in the current study.
\end{remark}

\section{Convergence analysis}\label{s:convergence analysis}

Convergence analysis of the DR algorithm is presented in this section.
Our goal is to establish a global convergence criterion without imposing any unverifiable mathematical assumption.
While the analysis scheme is standard (see, \textit{e.g.}, the book \cite[Chapter 10]{Bec17}), proving that the physical properties of the NDR problem under consideration indeed fulfil all the mathematical conditions required for convergence (Lemma \ref{l:smooth of f}) is the main contribution of this section.

\subsection{Convergence results from convex optimization}

In this section $\Ebb$ is a finite dimensional Euclidean space, $\dom(\cdot)$ denotes the domain of a function, and $\intr(\cdot)$ denotes the interior of a set. We consider the following composite optimization model
\begin{equation}\label{OP}
	\min_{x\in \Ebb} F(x) := f(x) + g(x) 
\end{equation}
under the assumptions specified below.
\begin{assumption}[mathematical assumptions]\label{ass:convex}\ 
	\begin{enumerate}[\rm (i)]
		\item\label{ass:i} $g:\Ebb \to (-\infty,\infty]$ is proper closed and convex.
		\item\label{ass:ii} $f:\Ebb \to (-\infty,\infty]$ is proper closed and convex, $\dom(f)$ is convex, and $\dom(g) \subset \intr\paren{\dom(f)}$.
		\item\label{ass:iii} $f$ is $\Gamma$-\emph{smooth} on $\intr\paren{\dom(f)}$, \textit{i.e.}, $f$ is differentiable on $\intr\paren{\dom(f)}$ and $\forall x,y\in \intr\paren{\dom(f)}$:
		\begin{equation}\label{L_f smooth}
			\norm{\nabla f(x) -\nabla f(y)} \le \Gamma \norm{x-y}.
		\end{equation}
		\item\label{ass:iv} The solution set of \eqref{OP}, denoted by $S$, is nonempty.
	\end{enumerate}
\end{assumption}

In view of \ref{ass:iv} the optimal value of \eqref{OP}, denoted by $F_{\mathtt{opt}}$, is attainable and finite. The following result serves as the basis for our subsequent analysis \cite[Theorem 10.34]{Bec17}.

\begin{theorem}[convergence of FISTA]\label{t:FISTA convergence}
	Suppose that all the conditions of Assumption \ref{ass:convex} hold true.
	Let $x^{(k)}$ be a sequence generated by FISTA for solving \eqref{OP} with stepsize $1/\Gamma$.
	Then for any $x^*\in S$,
	\begin{equation}\label{O(1/k^2)}
		F(x^{(k)}) - F_{\mathtt{opt}} \le \frac{2\Gamma}{(k+1)^2}\norm{x^{(0)} - x^*} \quad\forall k\ge 1.
	\end{equation} 	
\end{theorem}

\subsection{Convergence of the DR algorithm}

To apply the convergence criterion formulated in Theorem \ref{t:FISTA convergence} to the DR algorithm, the NDR problem (\ref{NDRP}) should be cast in the framework of (\ref{OP}) and all the requirements specified in Assumption \ref{ass:convex} need to be verified. Indeed, let us define $\mathbb{E} := \mathbb{R}^{(l, w)}$ and $g:=\iota_\Omega$ the indicator function of $\Omega= [0,1]^{(l,w)}$ where $(l,w)$ is the size of the object.
Then the \emph{proximal mapping} of $g$ is exactly the \emph{projector} on $\Omega$ characterized in Lemma \ref{l:P_Ocal} and FISTA applied to \eqref{OP} reduces to the DR algorithm applied to (\ref{NDRP}). The following lemma allows us to fully apply the convergence criterion in Theorem \ref{t:FISTA convergence} to the DR algorithm. The proof is presented in Appendix \ref{a:smoothness of f}.

\begin{lemma}\label{l:smooth of f}
	Regarding the NDR problem \eqref{NDRP}, the following statements hold true.
	\begin{enumerate}[\rm (i)]
		\item\label{as:i} $\iota_\Omega$ is proper closed and convex.
		\item\label{as:ii} $f$ is proper closed and convex, $\dom(f)$ is convex, and $\dom(\iota_\Omega) \subset \intr\paren{\dom(f)}$.
		\item\label{as:iii} $f$ is $\Gamma$-smooth on $\intr\paren{\dom(f)}$ with $\Gamma=\rho$.
		\item\label{as:iv} The solution set, denoted by $S$, is nonempty.
	\end{enumerate}
\end{lemma}

We can now formulate the desired convergence result for the DR algorithm.

\begin{theorem}[$O(1/k^2)$ - convergence rate of the DR algorithm]\label{t:NDR convergence}
	Let $o^{(k)}$ be a sequence generated by the DR algorithm for solving \eqref{NDRP} with stepsize $\lambda = \rho^{-1}$. Then for any $x^*\in S$,
	\begin{equation*}
		f(o^{(k)}) - F_{\mathtt{opt}} \le \frac{2\rho}{(k+1)^2}\norm{o^{(0)} - x^*} \quad\forall k\ge 1.
	\end{equation*}
\end{theorem}
\begin{proof}
	Lemma \ref{l:smooth of f} ensures that the NDR problem (\ref{NDRP}) satisfies all the mathematical conditions in Assumption \ref{ass:convex} with $\Gamma = \rho$. The result then follows from Theorem \ref{t:FISTA convergence}. 
\qed
\end{proof}

\section{Simulation results}\label{s:numerical results}

In this section we study important theoretical aspects of the DR algorithm on simulation data including solvability, noise robustness, convergence, model sensitivity and computational complexity.
Since we have not found existing single-frame nonuniform defocus removal algorithms that can be directly applied to the image classification problem considered in this paper, only few comparisons will be done in this section to show the advantages of the DR algorithm over the projected gradient method and the direct implementation of FISTA when relevant. Note that the analysis in Sect. \ref{s:solution algorithm} allows us to implement almost every standard optimization algorithm for the NDR problem under consideration. However, demonstrating the advantages of the DR algorithm over the others is not a goal of this paper as the efficiency of FISTA has been well known.
In the section we make use of the forward imaging model (\ref{forward model}).
Every (continuous) curve plotted in this paper is the linear interpolation of the discrete data being analysed.

\subsection{Solvability analysis}\label{subs:solvability}

\begin{table}[H]
	\caption{Parameters used in solvability analysis}
	\label{tbl:solvability params}
	\centering{
		\begin{tabular}[1\baselineskip]{|c|c|c|c|c|c|c|c|}
			\hline
			$N$ & $s$ & $d$ & $n_0$ & $\rho$ & $l=w$ & $\lambda=t^{(0)}$ & $K$\\ \hline
			141 & 3 & \textbf{0.1 -- 0.45} & 71 & 65 & 423 & 1 & 250\\ \hline
		\end{tabular}\vspace*{.25cm}\\
		$N$ -- \# defocus zones, $s$ -- DoF size (in pixel rows), $d$ -- blur coefficient, $n_0$ -- focal position, $\rho$ -- PSF size, $(l,w)$ -- object size, $\lambda$ -- stepsize,\\
		$t^{(0)}$ -- initial acceleration parameters, and $K$ -- \# iterations.
	}
\end{table}

We study the solvability of the DR algorithm with respect to the degradation level of the input image which is determined by the \emph{blur coefficient}. The parameters used in this experiment are presented in Table \ref{tbl:solvability params}, \textit{cf.} the notions introduced in Sect. \ref{subs:application model} and Sect. \ref{s:solution algorithm}.
The data image is generated according to (\ref{forward model}) and then corrupted with Poisson noise using the \textit{MATLAB imnoise} function. The restoration quality is measured by the relative Root Mean Squares (RMS) error (in \%) of the corrected image compared to the ground truth shown in Fig. \ref{fig:solvability}(a). Numerical experiments with 15 values of blur coefficient ranging from 0.1 to 0.45 and split evenly at 0.025 are conducted. Accordingly, the errors of the input images with noise gradually increase from $10.3\%$ to $20.3\%$ as shown in Fig. \ref{fig:solvability}(b) by the dark curve in comparison with the errors of the corrected images which range from $0.37\%$ to $4.8\%$ (the blue curve).
As expected the restoration error is directly proportional to the blur coefficient as depicted by the upward blue curve. For example, the data image with error $18.5\%$ and its restoration with error $3.2\%$ are respectively shown in figure (c) \& (d) for visualization. We observe that most details of the restored image shown in (d) are visually recognizable.

\begin{figure}[ht]
	\centering
	\includegraphics[width=.43\linewidth]{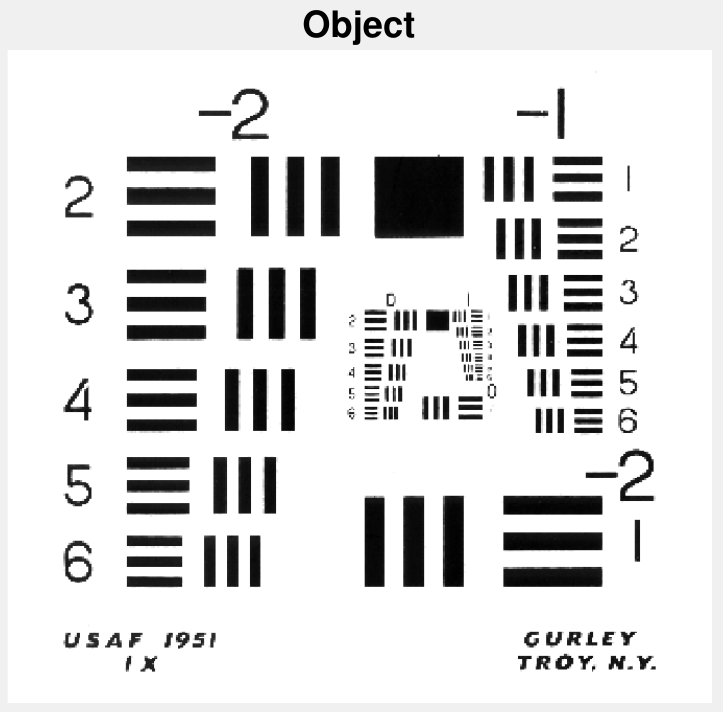}(a)\;
	\vspace*{0.25cm}
	\includegraphics[width=.43\linewidth]{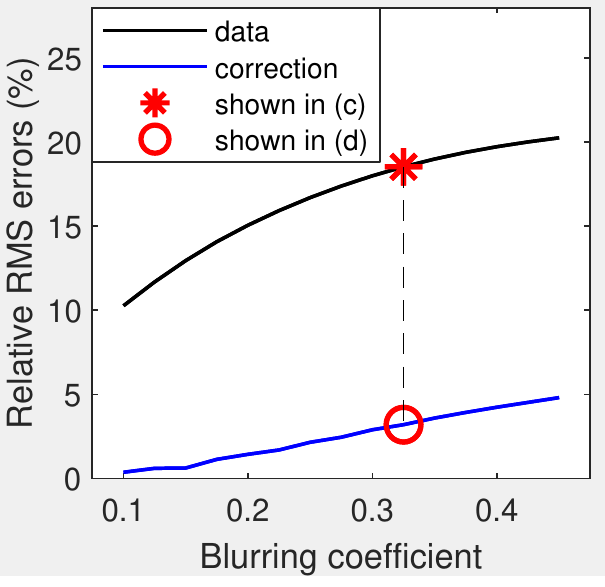}(b)\\
	\includegraphics[width=.43\linewidth]{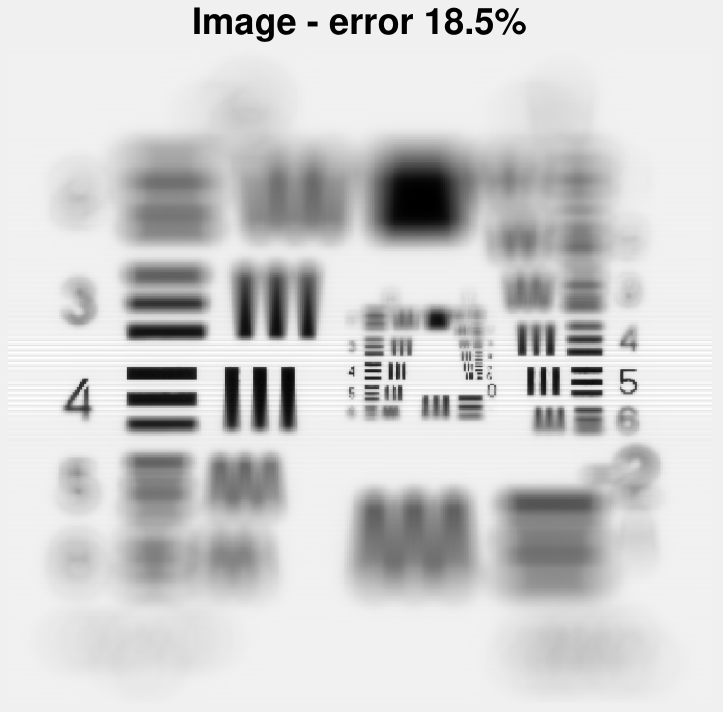}(c)\;
	\includegraphics[width=.43\linewidth]{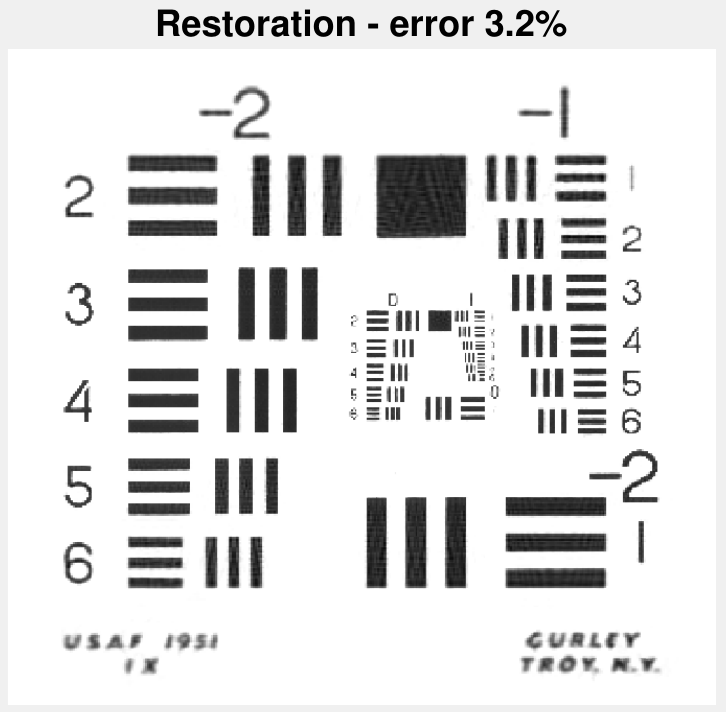}(d)
	\caption{Solvability analysis of the DR algorithm with respect to the degradation level of the data image measured by the \textit{blur coefficient}: (a) the ground true object; (b) the restoration errors (the blue curve) are compared to errors of the input images (the dark curve) for blur coefficients ranging from 0.1 to 0.45. The restoration error is directly proportional to the blur coefficient as depicted by the upward blue curve. (c) The input image with error $18.5\%$ versus (d) the corrected image with visible details (error $3.2\%$).}
	\label{fig:solvability}
\end{figure}

\subsection{Noise analysis}\label{subs:noise}

We analyse the robustness of the DR algorithm with respect to noise. The parameters used this experiment are as in Table \ref{tbl:solvability params} except with blur coefficient $d=0.1$. Numerical experiments with additive white Gaussian noise with eleven signal-to-noise ratios (SNR) ranging from 30 to 60 decibels (dB) are conducted. The restoration is terminated either after 150 iterations or whenever the estimate quality in terms of relative RMS drops by a threshold ($10^{-4}$ is taken). In this section we also compare the DR algorithm with the projected gradient (PG) method.
The results are summarized in Fig. \ref{fig:noise_analysis}(a) where the error of the corrected image is plotted versus the SNR. For SNR increasing from 30 to 60dB, the restoration error of the DR algorithm (the blue curve) gradually reduces from 6.4\% to 0.9\%.
It is worth mentioning that the higher error of PG (the red curve) compared to DR does not mean that the former is less robust than the latter.
This is rather due to the fact that PG would require much more iterations to reach the corresponding accuracy level of DR, see Sect. \ref{subs:convergence} for more details about their convergence speed.
Fig. \ref{fig:noise_analysis}(b) shows a section of the data\footnote{This section exhibits the most defocus in the image.} and the corresponding ones corrected by PG (the middle column) and DR (the last column) for six SNRs.
For SNR from 36dB the restorations by both methods are visually recognizable and the one by DR is obviously more visible than the one by PG.

\begin{figure}[h]
	\centering
	\includegraphics[width=.43\linewidth]{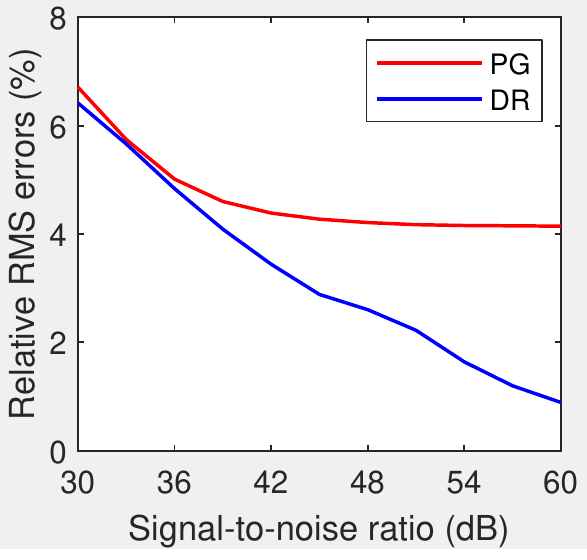}(a)\;
	\includegraphics[width=.43\linewidth]{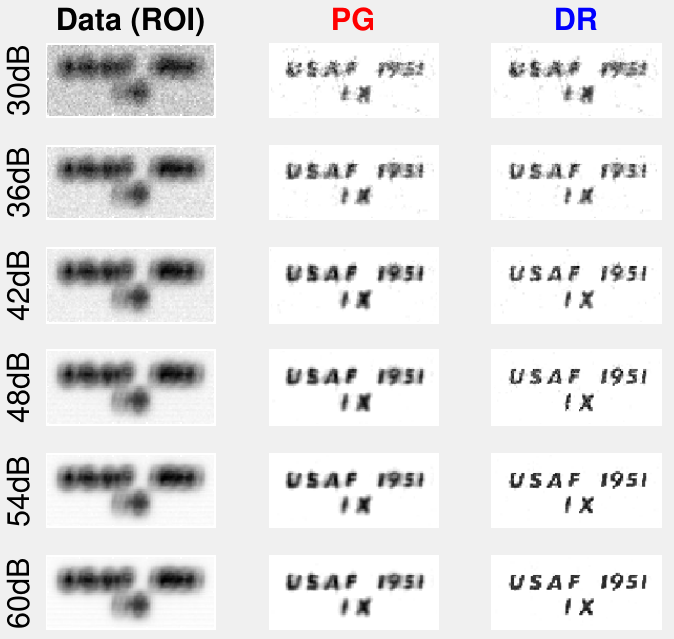}(b)
	\caption{Noise analysis of the DR algorithm with a comparison to the projected gradient (PG) method. (a) The correction errors for different levels of Gaussian noise with SNR ranging from 30 to 60dB are plotted. The error of DR (the blue curve) gradually reduces from 6.4\% to 0.9\% which is lower than that of PG (the red curve). 
		(b) A section of the data (the left column) and the corresponding ones corrected by PG (the middle column) and DR (the last columns) for six particular noise levels are shown for visual comparison. For SNR from 36dB the images corrected by both methods are visually recognizable and the one by DR is clearer.}
	\label{fig:noise_analysis}
\end{figure}

\subsection{Convergence properties}\label{subs:convergence}

In this section convergence properties of the DR algorithm are numerically demonstrated and also compared to those of the PG method. The parameters used in this experiment are as in Table \ref{tbl:solvability params} except with $d=0.1$ and $K=2500$. The data is corrupted with Poisson noise.
The numerical results are summarized in Fig. \ref{fig:convergence} in which the $y$-axis quantity is shown in the logarithmic scale for clarity.
The iterative change of the temporal estimate is shown in Fig. \ref{fig:convergence}(a) where the slopes of the tailing parts of the curves show linear convergence of the two algorithms. The DR algorithm (the blue curves) converges much faster than the PG method (the red curves) as expected. The restoration error in iteration is plotted in Fig. \ref{fig:convergence}(b) where its steady decrease shows that the convergence of both algorithms is towards stable and meaningful solutions. DR is also more accurate than PG, for example, to achieve a restoration error $<1\%$, 2500 iterations of PG is needed while that of DR is only about 100. The restoration error of 2500 DR iterations is about 130 times smaller than the one of PG.
For the sake of brevity, we chose to skip the numerical results in the noiseless setting which consistently show linear convergence of the DR algorithm to the correct solution.

\begin{figure}[ht]
	\centering
	\includegraphics[width=.425\linewidth]{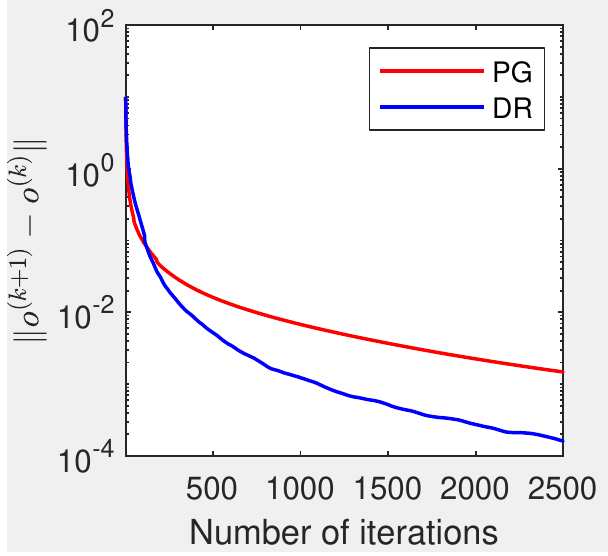}(a)\;
	\includegraphics[width=.43\linewidth]{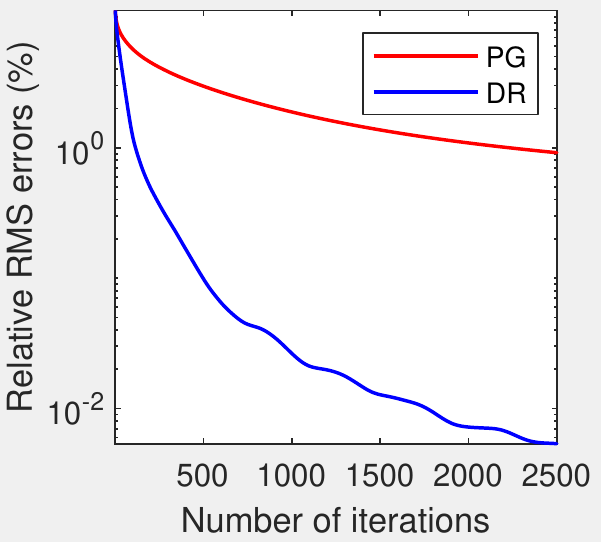}(b)
	\caption{The change (a) and the error (b) of the temporal estimate by the DR algorithm (the blue curves) are plotted in comparison with the PG method (the red curves). Figure (a) shows that DR converges faster than PG and the convergence of both methods is linear as indicated by the slopes of the tailing parts of the curves. Steady decrease of the restoration error in figure (b) shows that convergence of the algorithms is towards stable and meaningful solutions. DR is faster and more accurate than PG, for example, to obtain an estimate with error $<1\%$, 2500 iterations of PG is needed while that of DR is about 100.}
	\label{fig:convergence}
\end{figure}

\subsection{Model sensitivity}\label{subs:sensitivity}

Throughout this paper the \textit{focal position} and the \textit{blur coefficient} of the data image are assumedly given by the camera specifications and photographic settings, see Sect. \ref{subs:application model}. However, they may not be known precisely in practice.
In this section we analyse the sensitivity of the DR algorithm with respect to these parameters. The parameters used in this experiment are as in Table \ref{tbl:solvability params} except with blur coefficient $d=0.125$. The data is also corrupted with Poisson noise.

\textit{Focal position.} The generated image is respectively corrected using the DR algorithm but with nine levels of \textit{focal offset} uniformly ranging from $-2$ to $+2$ (DoF).
The numerical results are summarized in Fig. \ref{fig:model sensitivity}(a) where the restoration error is plotted versus the offset level.
The error is approximately linear to the offset level in both directions as depicted by the two almost straight branches of the curve.
The smallest error $0.6\%$ corresponds to the case without offset while the largest residual about $4.45\%$ occurs to the one with the most offset.

\begin{figure}[ht]
	\centering
	\includegraphics[width=.43\linewidth]{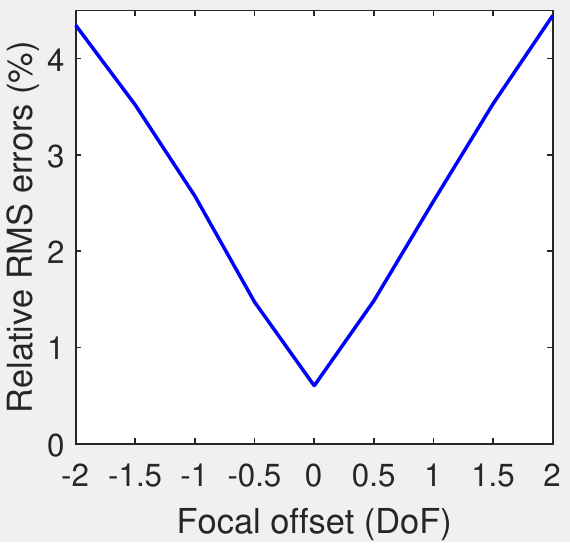}(a)\;
	\includegraphics[width=.43\linewidth]{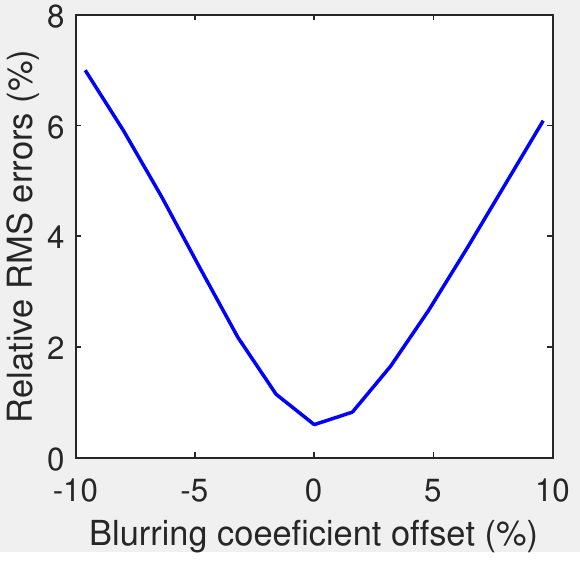}(b)
	\caption{(a) Sensitivity analysis of the DR algorithm with respect to the focal position: the restoration error is plotted versus the focal offset ranging from $-2$ to 2 (DoF). The error is approximately linear to the offset level in both directions. (b) Sensitivity analysis of the DR algorithm with respect to the blur coefficient: the restoration error is plotted versus the blur coefficient offset ranging from $-9.6\%$ to $9.6\%$. The restoration quality deteriorates proportionally relative to the offset level in both directions.}
	\label{fig:model sensitivity}
\end{figure}

\textit{Blur coefficient.} The generated image is respectively corrected using the DR algorithm but with 13 levels of blur coefficient offset uniformly ranging from $-0.012$ to $+0.012$, accordingly from $-9.6\%$ to $+9.6\%$.
The numerical results are summarized in Fig. \ref{fig:model sensitivity}(b) where the restoration error is plotted versus the offset level. The restoration quality deteriorates proportionally relative to the offset level in both directions.
The smallest error $0.6\%$ corresponds to the case without offset while the largest error about $7\%$ occurs to the one with the most offset.

In view of Fig. \ref{fig:solvability}(b), relative RMS error about $3.2\%$ can be considered as acceptable with respect to the eyeball metric. According to this criterion, we can conclude from the above analysis (summarized in Fig. \ref{fig:model sensitivity}) that the DR algorithm is robust with respect to focal offset up to 1.5 DoFs and blur coefficient offset up to $5\%$. In Sect. \ref{subs:exp_params} we propose numerical approaches for estimating these parameters using guide-star objects. The numerical results with experimental data in Sect. \ref{subs:exp_results} also confirm the robustness of the DR algorithm with respect to these parameters.

\subsection{Computational efficiency of DR}\label{subs:complexity}

We briefly report the advantage in computing time of the DR algorithm over the direct implementation of FISTA (Algorithm \ref{al:FISTA}). Note that evaluating the computational complexity of these algorithms in terms of flopping counts or else is not a goal of this section though this task is rather straightforward. The parameters used in this experiment are as in Table \ref{tbl:solvability params} except with blur coefficient $d=0.125$. Additionally taking the sparsity of the masks $\mu_n$ into account, the DR algorithm is about 4.7 times faster than Algorithm \ref{al:FISTA}.

\section{Experimental results}\label{s:exp_data}

In this section we aim at experimentally validating the forward imaging model (\ref{forward model}) of the nonuniform defocus removal problem considered in this paper and demonstrating the practicability of the proposed solution algorithm.

\subsection{Experiment setup}\label{subs:exp_setup}

We first describe the imaging setup of our experiment which has been done as simple and independent of the camera specifications as possible. A text object on a piece of paper is imaged by a camera with optical axis not perpendicular to the paper plane.
To make it convenient for partitioning the pictorial object into isoplanatic zones, before taking the image we manually rotated the camera around its axis so that each column of the object (approximately) corresponds to a single defocus value. This adjustment is equivalent to rotating the obtained image around its centre at the same angle but with opposite direction. In the current experiment the rotation angle can be numerically estimated as explained in Sect. \ref{subs:exp_params} below, however, adjusting the camera in advance helps avoid possible inexactness of numerically rotating the image posteriorly.

\begin{figure}[ht]
	\centering
	\includegraphics[width=0.75\linewidth]{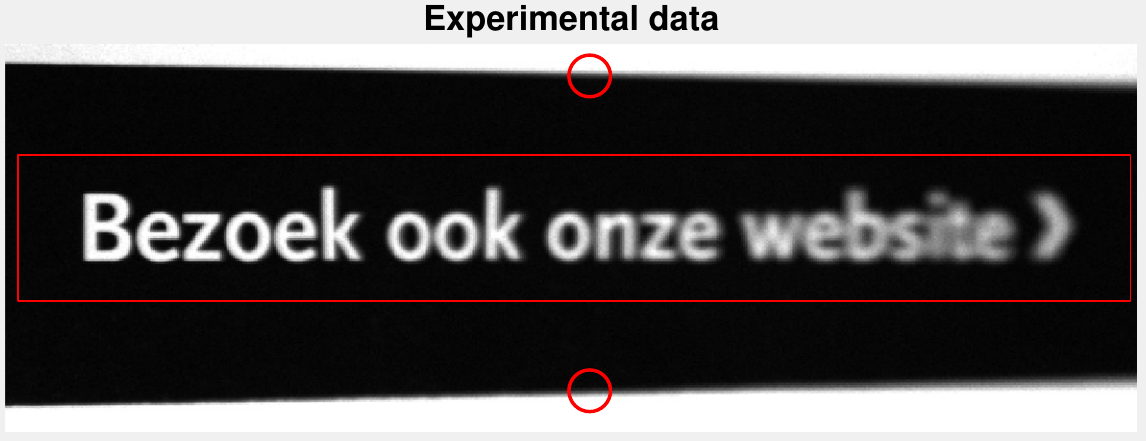}
	\caption{The experimental data consists of the image of the ROI in the central part marked by the rectangle and the \textit{guide-stars} marked by the red circles which are the images of two parallel straight sharp edges.}
	\label{fig:exp_img}
\end{figure}

The experimental data is shown in Fig. \ref{fig:exp_img} where the region of interest (ROI) is the text in its central part marked by the red rectangle. The text object is intentionally chosen so that the ROI is additionally accompanied with two parallel straight sharp edges whose images marked by the red circles in the upper and lower parts of Fig. \ref{fig:exp_img} are used as the \textit{guide-stars} of the imaging setup. As detailed in the next section, they enable us to estimate the parameters of the imaging model (\ref{forward model}) without requiring the knowledge of the camera specifications and photographic settings.
The idea of making use of the \textit{a priori} known \textit{salient edges} for such parameter estimation purposes can be found, \textit{e.g.}, in \cite{Dong2017, LiuYanZen19}.
According to the terminologies in \cite{LiuYanZen19}, the problem with experimental data studied in this section can be termed as \textit{surface-aware} blind nonuniform defocus removal.

\subsection{Estimate of physical parameters}\label{subs:exp_params}

In order to match the experimental data with the imaging model (\ref{forward model}), we need to determine the focal position and the blur coefficient, \textit{cf.} Sect. \ref{subs:application model}. Though these physical parameters are technically known from the camera specifications and the tilt angle of the camera axis relative to the object plane, they are unfortunately not at our disposal. As an alternative, we propose \textit{ad-hoc} numerical approaches for estimating them from the two \textit{guide-stars} in the data.

\textit{Focal position.} Two in-focal points are sufficient for determining the in-focal section of the pictorial object, which is the intersection of the focal and the object planes. The idea is to find one in-focal point along each of the guide-stars. 
Since the camera has been adjusted so that isoplanatic sections of the object are vertical as explained in Sect. \ref{subs:exp_setup}, we only need to do that for one of the guide-stars (\textit{e.g.}, the lower one).  
More specifically, let the lower guide-star be the rectangle region of size $30\times 580$ pixels proportionately containing the lower edge shown in Fig. \ref{fig:exp_img}. To estimate the focal position along this guide-star, we compute its gradient as a measure of its sharpness. As the optical aberration of the camera can be neglected, the guide-star as the image of a straight homogeneous sharp edge should exhibit its highest sharpness at in-focal points. According to this criterion, the plot of the highest sharpness of individual columns of the guide-star (the cyan curve) in Fig. \ref{fig:exp_slope and df}(a) indicates that the focal is around the 40th column which is marked by the blue rectangle in the figure. This yields the focal position $n_0 \approx 8$ (in DoF) if the DoF size $s=5$ is chosen as what we do in the next section.

\textit{Blur coefficient.} The variation of the sharpness of the (lower) guide-star is mutually related to the blur coefficient of the imaging setup: the larger the blur coefficient, the greater the variation (around the focal).
This observation leads to the following \emph{ad-hoc} grid search for estimating the blur coefficient.
First, the sharpness variation of the guide-star is measured by the \emph{average slope} of the sharpness curve via linear approximation plotted by the red line in Fig. \ref{fig:exp_slope and df}(a). The sharpness curve is truncated between columns $40$ (around the focal position) and $500$ to reduce the influence of noise.
Next, a simulated straight sharp edge is passed through the imaging model (\ref{forward model}) with respectively 17 different values of blur coefficient ranging from 0.01 to 0.05 uniformly separated by 0.0025. The DoF size $s=5$ (pixels) is empirically chosen.
The sharpness variations of the obtained 17 images are then computed in the same way as for the one of the guide-star. The plot in Fig. \ref{fig:exp_slope and df}(b) confirms that the larger the blur coefficient, the greater the variation. The sharpness variation of the guide-star, marked in red, provides an \textit{ad-hoc} estimate of the blur coefficient about 0.0296. It is worth mentioning that for this experiment the DR algorithm is highly robust with respect to both the focal position and the blur coefficient.

\begin{figure}[ht]
	\centering
	\includegraphics[width=.42\linewidth]{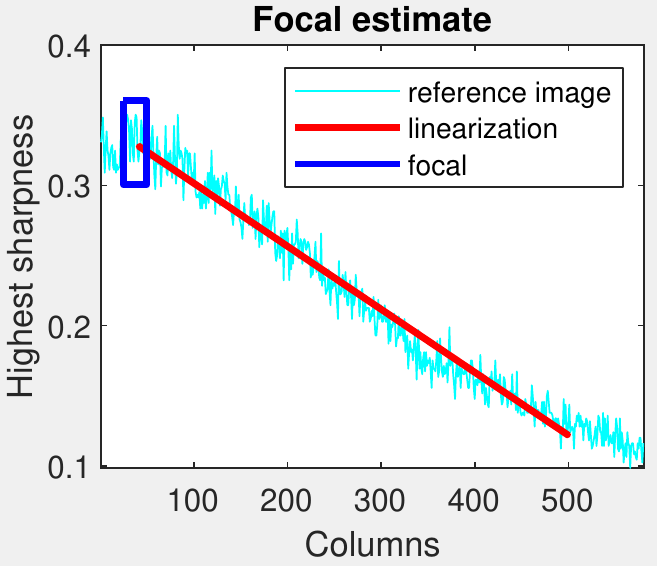}(a)\;\;
	\includegraphics[width=.435\linewidth]{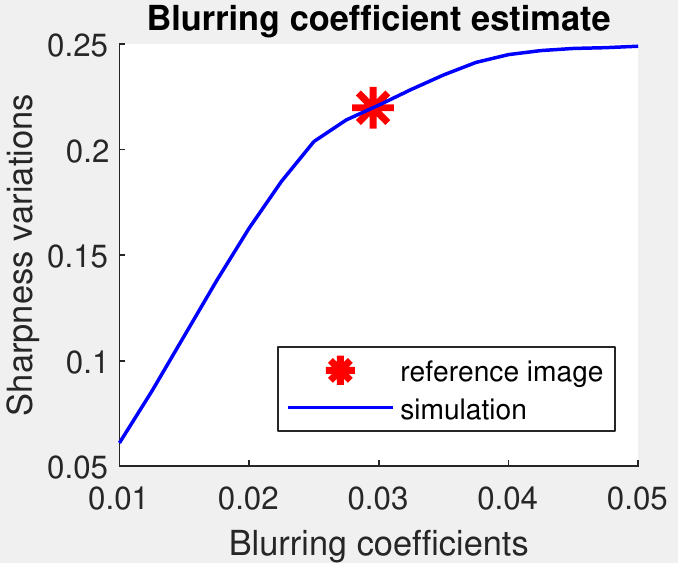}(b)
	\caption{(a) The in-focal point along the lower \textit{guide-star} is estimated from its sharpness. The cyan curve shows the highest sharpness of each column of the guide-star.
		The guide-star as the image of a straight homogeneous sharp edge would exhibit its highest sharpness at the in-focal point (around column 40) marked by the blue rectangle. The \emph{blur coefficient} of the imaging setup is estimated from the sharpness variation of the guide-star measured by the average slope of the sharpness curve (in cyan) via linear approximation (the red line). (b) The plot summarizes the \textit{ad-hoc} grid search over 17 different values of blur coefficient versus the sharpness variation levels: the larger the blur coefficient, the greater the variation. The sharpness variation of the guide-star marked in red provides an estimate of the blur coefficient about 0.0296.}
	\label{fig:exp_slope and df}
\end{figure}

\subsection{Image processing and results}\label{subs:exp_results}

\begin{table}[h]
	\caption{Parameters used in the experiment of Sect. \ref{s:exp_data}}
	\label{tbl:exp params}
	\centering{
		\begin{tabular}[1\baselineskip]{|c|c|c|c|c|c|c|c|c|c|c|}
			\hline
			$N$ & $s$ & $d$ & $n_0$ & $\rho$ & $l$ & $w$ & $\tau$ & $\lambda$ & $t^{(0)}$ & $K$\\ \hline
			116 & 5 & 0.0296 & 8 & 23 & 90 & 580 & 0.11 & 1 & 1 & 10\\ \hline
		\end{tabular}\vspace*{.25cm}
	}
	$N$ -- \# defocus zones, $s$ -- DoF size (in pixels), $d$ -- blur coefficient,\\
	$n_0$ -- focal position, $\rho$ -- PSF size, $(l,w)$ -- object size, $\tau$ -- data threshold,\\
	$\lambda$ -- stepsize, $t^{(0)}$ -- initial acceleration parameters, and $K$ -- \# iterations.
\end{table}

Let us first explain more quantitative details of the numerical experiment.
The ROI (in the object) is assumed to have the same size as its image shown in Fig. \ref{fig:exp_img} which is $90\times 580$. The depth of field $s=5$ is taken.
That is, the ROI consists of 116 equal vertical zones each consisting of 5 adjacent columns. Following the analysis in Sect. \ref{subs:exp_params}, we set the required physical parameters as follows: the 8th zone is in focal and the blur coefficient is 0.0296. It is more intuitive to conceive the blur coefficient via the corresponding generated PSFs, for example, the one with most defocus for the rightmost zone of the ROI is shown in Fig. \ref{fig:exp_his and psf}(a). The ROI is restored by 10 iterations of the DR algorithm with stepsize $1$ and the initial acceleration parameter $1$. The parameters used in this experiment are also summarized in Table \ref{tbl:exp params}.

\begin{figure}[H]
	\centering
	\ \ \includegraphics[width=.395\linewidth]{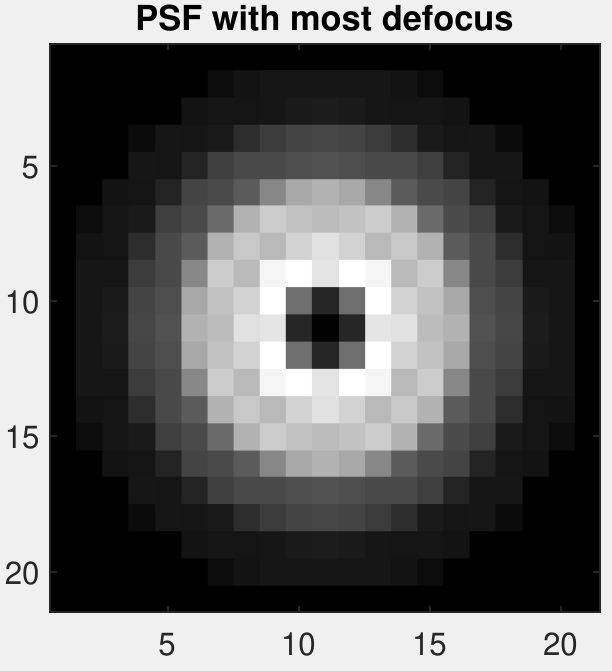}(a)\;\;
	\includegraphics[width=.4\linewidth]{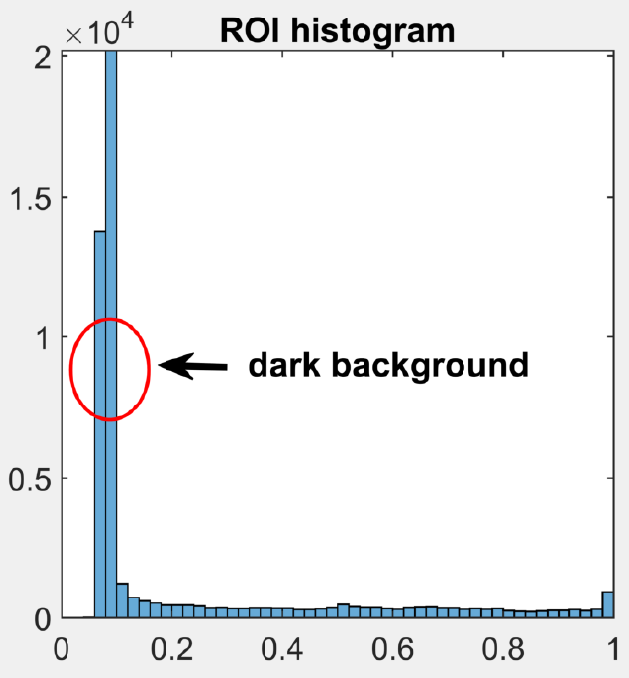}(b)
	\caption{(a) To intuitively illustrate the blur coefficient in the experiment data, the PSF for the rightmost zone of the ROI generated using the physical parameters given by Sect. \ref{subs:exp_params} is shown. (b) The histogram of the image of the ROI reveals the information about the background intensities of the ROI (most of its pixels are almost dark). This knowledge is helpful for preprocessing the experimental data, \textit{e.g.}, thresholding the data at 0.11 results in better image restoration (without artifacts) as shown in Fig. \ref{fig:exp_est}(c) compared to Fig. \ref{fig:exp_est}(b).}
	\label{fig:exp_his and psf}
\end{figure}

Due to noise and inevitable model deviations, direct use of the experimental data results in undesired artifacts in the restored image as shown in Fig. \ref{fig:exp_est}(b). These artifacts do not spoil down the overall improvement of the corrected image compared to the data in Fig. \ref{fig:exp_est}(a), but they are still visible in the right most area of Fig. \ref{fig:exp_est}(b). To address this issue we make use of the information given by the image histogram shown in Fig. \ref{fig:exp_his and psf}(b) which indicates that the background intensities of the ROI fall in the interval $[0.06,0.1]$ and the data is quite noisy. This gives us a hint for thresholding the data at 0.11 (manually and empirically chosen). This simple manipulation indeed solves the artifact issue as shown in Fig. \ref{fig:exp_est}(c) where the restoration with the threshold data no longer exhibits distortion effects as opposed to Fig. \ref{fig:exp_est}(b). The improvement in image quality of the correction in Fig. \ref{fig:exp_est}(c) compared to the data in Fig. \ref{fig:exp_est}(a) is clearly observable. This proves both the physical relevance of the NDR problem (\ref{NDRP}) and the practicability of the DR algorithm.

\begin{figure}[H]
	\centering
	\ \ \includegraphics[width=.75\linewidth]{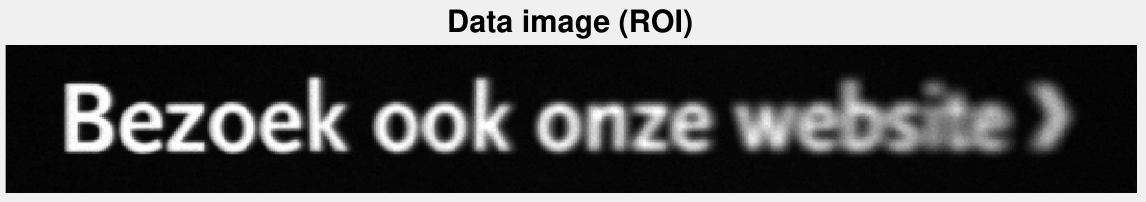} (a)\\
	\ \\
	\ \ \includegraphics[width=.75\linewidth]{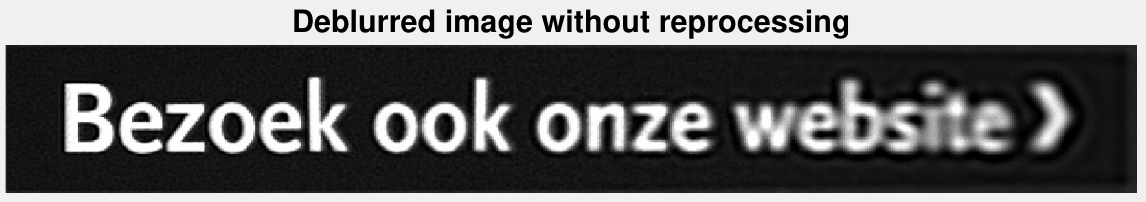} (b)\\
	\ \\
	\ \ \includegraphics[width=.75\linewidth]{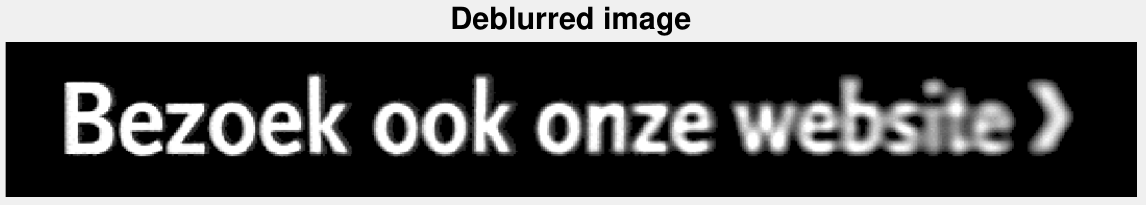} (c)
	\caption{(a) The image of the ROI from the experimental data. (b) The restored ROI with \emph{unprocessed} data suffers artifacts which are slightly visible in its rightmost area. (c) The restored ROI with the data threshold at 0.11 no longer exhibits artifacts compared to (b). The improvement of the correction (c) compared to the data (a) proves both the physical relevance of the NDR problem (\ref{NDRP}) and the practicability of the DR algorithm.}
	\label{fig:exp_est}
\end{figure}

\section{Application to image classification}\label{s:image classification}

In this section we demonstrate that the proposed method for nonuniform defocus removal has potential for practical applications by considering the image classification problem using convolutional neural networks.
The research question is that having a reliable classification neural network trained with standard (undistorted) images, whether or not we can efficiently classify distorted images caused by nonuniform defocus. For example, we want to apply a certain fault detection neural network to similar manufacturing lines whose computer vision system, however, has different perspective geometry.
Recall that our approach to this problem is to introduce an additional image preprocessing step without modifying the given classification neural network.

\begin{table}[h]
	\caption{Parameters used to generate the distorted test datasets B(0.2) and B(0.3) according to the forward model (\ref{forward model})}
	\label{tbl:distorted_test_cifar10_params}
	\centering{
		\begin{tabular}[1\baselineskip]{|c|c|c|c|c|c|c|c|c|c|c|c|}
			\hline
			$N$ & $s$ & $d$ & $n_0$ & $\rho$ & noise\\ \hline
			64 & 1 & 0.2 \& 0.3 & 16 & 9 & Poisson\\ \hline
		\end{tabular}\vspace*{.25cm}\\
		$N$ -- \# defocus zones, $s$ -- DoF size, $d$ -- blur coefficient,\\
		$n_0$ -- focal position, and $\rho$ -- PSF size.
	}
\end{table}

\begin{table}[ht]
	\caption{Accuracy rates of the considered image classification neural network applied to the original test dataset (undistorted images), the distorted ones B(0.2) \& B(0.3) and the corrected ones using the DR algorithm C(0.2) \& C(0.3)}
	\label{tbl:test_cifar10_summary}
	\centering{
		\begin{tabular}[1\baselineskip]{|c|c|c|c|c|c|}
			\hline
			\textbf{Test dataset} & Original & B(0.2) & C(0.2) & B(0.3) & C(0.3) \\ \hline
			\textbf{Accuracy rate} (\%) & 70.09 & 50.67 & \textbf{65.35} & 42.24 & \textbf{56.46} \\ \hline
		\end{tabular}\vspace*{.25cm}\\
		Accuracy rates for C(0.2) and C(0.3) are 65.35\% and 56.46\% which are 14.68\% and 14.22\% higher than the ones for B(0.2) and B(0.3), respectively.
	}
\end{table}

We study the CIFAR-10 dataset consisting of 60 thousands colour images of size $32\times 32\times 3$ in ten classes of objects with 6 thousands images in each class \cite{cifar10}.
The dataset is mutually exclusively divided into 50 thousands training images and 10 thousands test ones.
In this example we make use of the convolutional neural network architecture consisting of three convolutional, two max pooling, one flatten and two fully connected layers built and documented by TensorFlow.
The classification algorithm is obtained by training the network for ten epochs each with batch size 32.
The accuracy of the obtained neural network applied to the 10 thousands test images, which we referred to as the \textit{original} test dataset in the sequel, is about 70.09\%.\footnote{This number as well as all presented in Table \ref{tbl:test_cifar10_summary} are slightly variant for each running since the images used in training the network (ten epochs) are randomly chosen from the training dataset. However, the overall performance of the algorithm over the considered datasets is consistent.}

We now assume that we want to use this neural network to classify images which are distorted by nonuniform defocus.
For this analysis we first create simulation test datasets by passing each of the 10 thousands test images through the forward imaging model (\ref{forward model}) with the technical parameters summarized in Table \ref{tbl:distorted_test_cifar10_params}.
The original test images are padded with constant boundary values to size $64\times 64\times 3$ (\textit{i.e.}, with padding size 16).
The \textit{full-size} images generated by (\ref{forward model}) with PSF size $9\times 9$ having size $72\times 72\times 3$ are then corrupted with Poisson noise and the $(32\times 32\times 3)$-central parts of the obtained images are taken as the distorted images.
Using this routine, we generate two sets of distorted images with blur coefficients $0.2$ and $0.3$, which are denoted by B(0.2) and B(0.3), respectively. Ten realizations of each are shown in Fig. \ref{fig:CNN_images_1} 
(the 2nd \& 4th columns).
As expected, direct application of the classification network to the distorted images is significantly less accurate.
As shown in Table \ref{tbl:test_cifar10_summary} the accuracy rates of classifying B(0.2) and B(0.3) are respectively 50.67\% and 42.24\% compared to 70.09\% for the original test dataset.

In the remainder of this section we will demonstrate that the DR algorithm can be used to correct the distorted images and hence improve the accuracy of the classification neural network.
Each element of B(0.2) and B(0.3) is corrected by 20 iterations of the DR algorithm with stepsize $1$ and initial acceleration step $1$.
Technically, the $(64\times 64\times 3)$-central part of each \textit{full-size} image generated by (\ref{forward model}) and corrupted by Poisson noise (as described in the previous paragraph) is taken as the input of the DR algorithm.
Note that such input image suffers both noise and information loss due to image boundary cut-off.
The $(32\times 32\times 3)$-central part of the image restored by DR is then taken as the corrected image.
Corresponding to B(0.2) and B(0.3) we obtained the two datasets of corrected images respectively denoted by C(0.2) and C(0.3).
Ten realizations of each are shown in Fig. \ref{fig:CNN_images_1} (the 3rd and last columns) for visual comparison to the distorted datasets as well as the original one.

\begin{figure}[H]
	\centering
	\includegraphics[width=.9\linewidth]{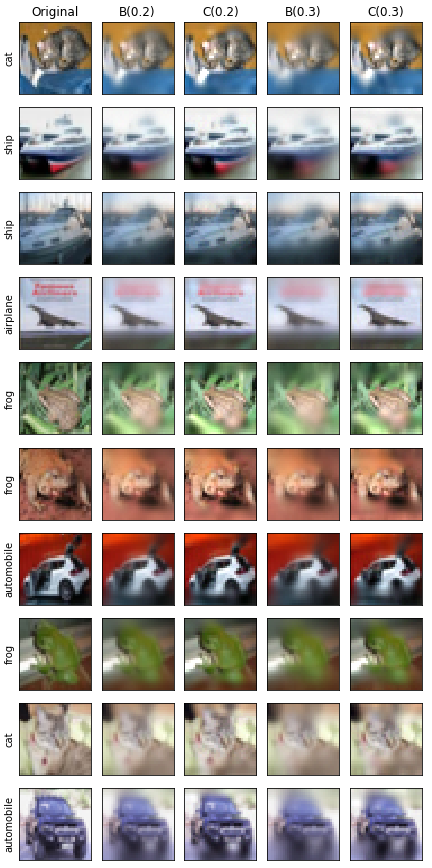}
	\caption{Realizations of various test datasets are shown for visual comparison: the 1st column -- the original test dataset with classification accuracy rate 70.09\%;\; the 2nd \& 4th columns -- the distorted ones B(0.2) and B(0.3) with rates 50.67\% and 42.24\%;\; the 3rd \& last columns -- the ones corrected using the DR algorithm C(0.2) and C(0.3) with rates \textbf{65.35\%} and \textbf{56.46\%}. The improved accuracy rates for C(0.2) and C(0.3) compared to the ones for B(0.2) and B(0.3) (with 14.68\% and 14.22\% higher, respectively) shows the effectiveness of the proposed image correction step and hence promising potential of the DR algorithm for practical image classification.}
	\label{fig:CNN_images_1}
\end{figure}

With respect to the eyeball metric, the visual quality of C(0.2) and C(0.3) is higher than that of B(0.2) and B(0.3) as can be clearly observed from Fig. \ref{fig:CNN_images_1}.
For the considered classification neural network, the accuracy rates for C(0.2) and C(0.3) are 65.35\% and 56.46\% which are respectively 14.68\% and 14.22\% higher than the ones for B(0.2) and B(0.3) as summarized in Table \ref{tbl:test_cifar10_summary}.
This improvement shows the effectiveness of the additional image correction step and hence promising potential of the DR algorithm for practical image classification.

\section{Conclusion}\label{s:conclusion}

We have investigated the \textit{single-frame nonuniform deblurring} problem associated with image classification using machine learning algorithms, named the NDR problem, and proposed the solution method called the DR algorithm for solving it. Important features of DR including solvability, noise robustness, convergence, model insensitivity and computational efficiency were demonstrated.
Physical relevance of the NDR problem and the practicability of the DR algorithm were tested on experimental data.
We also demonstrated that the proposed solution approach for nonuniform defocus removal indeed works for the target application in image classification which originally motivated the investigation of the NDR problem in this paper.

Several meaningful extensions of the NDR problem and the solution approach for it should be investigated in future research. First, the optical aberration of the camera is currently ignorable, see Sect. \ref{subs:NDB}. The question is how the solution algorithm should be modified for the opposite scenario, where the blurring kernels are in general not symmetric, \textit{cf.} the proof of Lemma \ref{l:smooth of f}.
Second, the three-dimensional geometric shape of the object has not been considered since it is unavailable in applications such as image classification. For applications with \textit{a priori} known reference, for example, in detecting errors/faults during the production of three-dimensional objects, it is meaningful to remove also the nonuniform defocus effects inherent from the object shape.
Third, imaging parameters such as \textit{focal position} and \textit{blur coefficient} currently assumed to be known (see Sect. \ref{s:problem formulation}) are often imprecise in practice. It would be meaningful if they could also be corrected by the deblurring algorithm.

\section*{Appendices}
\subsection*{Proof of Lemma \ref{l:nabla f}}\label{a:nabla f}
\begin{proof}
	Given the PSFs $p_n$ and the masks $\mu_n$ ($n=1,2,\ldots,N$), let us define the mapping $L:\mathbb{E}\to \mathbb{E}$ given by
	\[
	L(x) := \sum_{n=1}^{N}(\mu_n\odot x)*p_n\quad \forall x\in \mathbb{E}.
	\]
	Then $L$ is a linear operator as both the convolution and the elementwise multiplication operations are linear. Consequently, the function $f$ defined in (\ref{OP1}) is quadratic with respect to $x$ and thus it is differentiable everywhere with the gradient given by
	\begin{equation}\label{f2L}
		\nabla f(x) = \nabla \paren{\frac{1}{2}\norm{i-L(x)}^2} = L^*\paren{i-L(x)},\; \forall x\in \mathbb{E},
	\end{equation}
	where $L^*$ is the adjoint operator of $L$ which by elementary algebra manipulations is given by
	\begin{equation}\label{L*}
		L^*(y) = \sum_{n=1}^{N}\mu_n \odot (p_n^\dagger * y)\quad\forall y\in \mathbb{E},
	\end{equation}
	where $p_n^{\dagger}$ denotes the reflection of $p_n$ via its origin, \textit{i.e.},
	\begin{equation*}
		p_n^{\dagger}(\xi,\eta) := p_n(-\xi,-\eta)\quad (n=1,2,\ldots,N).
	\end{equation*}
	The combination of (\ref{f2L}) and (\ref{L*}) yields the gradient formula (\ref{nabla f}) as claimed.
\qed
\end{proof}

\subsection*{Parallel computing of $\nabla f$}\label{a:parallel computing}

As the convolution kernels are spatially varying, the evaluation of the gradient in Lemma \ref{l:nabla f} is computationally highly expensive and becomes a tremendous burden on implementing the gradient-type algorithms for practical applications.
Fortunately thanks to the rapidly increasing capacity of computational hardware enabling one to effectively handle larger-scale problems, the aforementioned drawback in terms of computational complexity can be overcome by exploiting parallel structure of the gradient.

The idea of parallel computing for convolution kernels is to rewrite the computationally expensive terms in (\ref{nabla f}) as multiplications of tensors and it is widely known that the latter mathematical operations can be computed much more efficiently with GPU implementation.

We first represent a sequence of convolution operations $x_n * p_n$ as a tensor-multiplication operation ($n=1,2,\ldots,N$).
Recall that the size of all the convolution kernels $p_n$ is $(\rho,\rho)$ and the size of the DoF zones $x_n$ is $(s,w)$. For each channel $n=1,2,\ldots,N$, the $x_n$ is symmetrically padded with zeros to yield the matrix of size $(s+2\rho-2)\times (w+2\rho-2)$.
By the linearity of the convolution operator, the obtained matrix can be one-to-one rewritten as the matrix $X_n$ of size $(s+\rho-1)(w+\rho-1)\times \rho^2$ satisfying the following equality:
\[
X_n \times \vecto(p_n) = \vecto(s_n:=x_n*p_n),
\]
where $\times$ is the matrix-multiplication operation and $\vecto(\cdot)$ is the vectorization operator.

One can stack all the transformed matrices $X_n$ to form a 3-order tensor of size $(s+\rho-1)(w+\rho-1)\times \rho^2 \times N$ denoted by $\mathcal{X}$, all the vectorized convolution kernels $\vecto(p_n)$ to form a 3-order tensor of size $\rho^2\times 1 \times N$ denoted by $\mathcal{P}$, and all the vectorized sub-images $\vecto(s_n)$ to form a 3-order tensor of size $(s+\rho-1)(w+\rho-1)\times 1 \times N$ denoted by $\mathcal{S}$.
Then the sequence of convolution operations $x_n*p_n$ ($n=1,2,\ldots,N$) can be computed as a single tensor-multiplication operation:
\[
\mathcal{X}\times \mathcal{P} = \mathcal{S}.
\]
This tensor-multiplication operation is the desirable format for GPU parallel computing for the sequence of convolution operations.

We next synthesize the 3-order tensor $\mathcal{S}$ to yield the inner sum in (\ref{nabla f}).
Let us construct a matrix $\mathcal{L}$ of size $(Ns+\rho-1) \times N(s+\rho-1)$ as follows: the $n$th $(s+\rho-1)$-column block of $\mathcal{L}$ contains all zero entries except a unit matrix of size $(s+\rho-1)$ in the rows $1:(s+\rho-1)+(n-1)s$.
Then the inner sum in (\ref{nabla f}) is given by the matrix-multiplication operation
\[
\mathcal{L} \times \conc(\vecto^{-1}(\mathcal{S})),
\]
where $\vecto^{-1}(\cdot)$ reshapes $\mathcal{S}$ to a 3-order tensor of size $(s+\rho-1)\times (w+\rho-1)\times N$ which is then concatenated to form an $N(s+\rho-1)\times (Ns+\rho-1)$-matrix.

Similarly parallel computing can also be deployed on the sequence of outer convolution operations in (\ref{nabla f}).

\subsection*{Proof of Lemma \ref{l:smooth of f}}\label{a:smoothness of f}
\begin{proof}
	\ref{as:i} The indicator function $\iota_\Omega$ is proper closed and convex as the set $\Omega$ is nonempty closed and convex, see, \textit{e.g.} \cite{VA}.
	
	\ref{as:ii} The function $f$ defined in (\ref{OP1}) is the composition of the squared norm and a linear function in $x$, which are both lower semicontinuous and convex. Thus $f$ is proper closed and convex \cite[Theorem 5.7]{Roc70}. Also, $\dom(f) =\mathbb{E}$ is convex and $\dom(\iota_\Omega) = \Omega \subset \intr\paren{\dom(f)}=\mathbb{E}$.
	
	\ref{as:iii} Note that as the optical aberration of the camera is neglected, all the PSFs are \emph{radial}, in particular, it holds that $p_n^{\dagger}=p_n$. Then by Lemma \ref{l:nabla f} $f$ is differentiable everywhere with the gradient given by, $\forall x\in \mathbb{E}$,
	\begin{equation*}
		\nabla f(x) = -\sum_{n=1}^{N}\mu_n\odot \paren{p_n*\paren{i - \sum_{m=1}^{N}\paren{\mu_m\odot x} * p_m}}.
	\end{equation*}
	Thus for any $x,y\in \mathbb{E}$, it holds that
	\[
	\nabla f(x) - \nabla f(y) = \sum_{n=1}^{N}\mu_n\odot \paren{p_n*r},
	\]
	where $r:=\sum_{m=1}^{N}\paren{\mu_m\odot (y-x)} * p_m$.
	Since the convolution operator is linear and $\norm{p_m}\le \norm{p_m}_1=1$, it holds that
	\begin{equation}\label{est_in}
		\begin{aligned}
			\norm{r}^2 &
			\le \sum_{m=1}^{N}\norm{\mu_m\odot (y-x)}^2 \norm{p_m}^2\\
			&\le \sum_{m=1}^{N}\norm{\mu_m\odot (y-x)}^2 = \norm{y-x}^2.
		\end{aligned}
	\end{equation}
	Let $r_n:= \mu_n\odot \paren{p_n*r}$ for each $n=1,2,\ldots,N$. Then its entries satisfy
	\begin{align*}
		|r_n(\xi,\eta)|^2 &= \left|\sum_{i,j=(\rho-1)/2}^{(\rho-1)/2} p_n(i,j)\, r(\xi+i,\eta+j)\right|^2\\
		&\le \norm{p_n}^2\sum_{i,j=(\rho-1)/2}^{(\rho-1)/2} r^2(\xi+i,\eta+j)\\
		&\le \sum_{i,j=(\rho-1)/2}^{(\rho-1)/2} r^2(\xi+i,\eta+j).
	\end{align*}
	As the last estimate is independent of individual PSFs and the pixels of $r$ play symmetric roles with respect to $n$ and coordinates $(\xi,\eta)$, we get
	\begin{equation}\label{est_out}
		\norm{\nabla f(x) - \nabla f(y)}^2 \le \rho^2\norm{r}^2.
	\end{equation}
	The combination of (\ref{est_in}) and (\ref{est_out}) yields that
	\[
	\norm{\nabla f(x) - \nabla f(y)} \le \rho\norm{r} \le \rho \norm{y-x}.
	\]  
	Consequently, $f$ is $\rho$-smooth on $\mathbb{E} = \intr\paren{\dom(f)}$ as claimed.
	
	\ref{as:iv} Since $f$ is lower semicontinuous and convex in view of \ref{as:ii} and the constraint set $\Omega$ is convex and compact, the solution set of the minimization problem (\ref{NDRP}) is nonempty closed convex and compact.
\qed
\end{proof}


%
%

\bibliographystyle{plain}       
\bibliography{shortbib}   

\begin{thebibliography}{10}

\bibitem{Bardsley2006}
Johnathan Bardsley, Stuart Jefferies, James Nagy, and Robert Plemmons.
\newblock {A computational method for the restoration of images with an
  unknown, spatially-varying blur}.
\newblock {\em Opt. Express}, 14(5):1767, 2006.

\bibitem{BecTeb09}
A.~Beck and M.~Teboulle.
\newblock A fast iterative shrinkage-thresholding algorithm for linear inverse
  problems.
\newblock {\em SIAM J. Imaging Sci.}, 2(1):183--202, 2009.

\bibitem{Bec17}
Amir Beck.
\newblock {\em First-{O}rder {M}ethods in {O}ptimization}, volume~25 of {\em
  MOS-SIAM Series on Optimization}.
\newblock Society for Industrial and Applied Mathematics (SIAM), Philadelphia,
  PA; Mathematical Optimization Society, Philadelphia, PA, 2017.

\bibitem{BooNeiJusWil02}
M.~J. Booth, M.~A.~A. Neil, R.~Ju\u{s}kaitis, and T.~Wilson.
\newblock Adaptive aberration correction in a confocal microscope.
\newblock {\em Proc. Natl. Acad. Sci.}, 99:5788--5792, 2002.

\bibitem{BoyVan04}
Stephen Boyd and Lieven Vandenberghe.
\newblock {\em Convex {O}ptimization}.
\newblock Cambridge University Press, Cambridge, 2004.

\bibitem{Can76}
M.~Cannon.
\newblock Blind deconvolution of spatially invariant image blurs with phase.
\newblock {\em IEEE Transactions on Acoust. Speech Signal Process.}, 24:58--63,
  1976.

\bibitem{ChaWuTsa17}
C.-F. Chang, J.-L. Wu, and T.-Y. Tsai.
\newblock A single image deblurring algorithm for nonuniform motion blur using
  uniform defocus map estimation.
\newblock {\em Mathematical Problems in Engineering}, 3:1--14, 2017.

\bibitem{DaiFie87}
J.~C. Dainty and J.~R. Fienup.
\newblock Phase retrieval and image reconstruction for astronomy.
\newblock {\em Image Recovery: Theory Appl.}, 13:231--275, 1987.

\bibitem{Dong2017}
Jiangxin Dong, Jinshan Pan, and Zhixun Su.
\newblock {Blur kernel estimation via salient edges and low rank prior for
  blind image deblurring}.
\newblock {\em Signal Process. Image Commun.}, 58:134--145, 2017.

\bibitem{Endo2020}
Kazuki Endo, Masayuki Tanaka, and Masatoshi Okutomi.
\newblock {Classifying Degraded Images Over Various Levels Of Degradation}.
\newblock In {\em 2020 IEEE Int. Conf. Image Process.}, volume~19, pages
  1691--1695. IEEE, 2020.

\bibitem{FenBou15}
W.~{Feng} and S.~{Boukir}.
\newblock Class noise removal and correction for image classification using
  ensemble margin.
\newblock In {\em 2015 IEEE International Conference on Image Processing
  (ICIP)}, pages 4698--4702, 2015.

\bibitem{FliRig05}
Ralf~C. Flicker and Fran\c{c}ois~J. Rigaut.
\newblock Anisoplanatic deconvolution of adaptive optics images.
\newblock {\em J. Opt. Soc. Am. A}, 22(3):504--513, 2005.

\bibitem{FosHun79}
F.~Foster and J.~W. Hunt.
\newblock Transmission of ultrasound beams through human tissue-focussing and
  attenuation studies.
\newblock {\em Ultrasound Medicine \& Biol.}, 5:257--268, 1979.

\bibitem{Goo05}
J.~W. Goodman.
\newblock {\em Introduction to {F}ourier {O}ptics}.
\newblock Roberts \& {C}ompany {P}ublishers, 2005.

\bibitem{Hirsh2010}
M.~{Hirsch}, S.~{Sra}, B.~{Schölkopf}, and S.~{Harmeling}.
\newblock Efficient filter flow for space-variant multiframe blind
  deconvolution.
\newblock In {\em 2010 IEEE Computer Society Conference on Computer Vision and
  Pattern Recognition}, pages 607--614, 2010.

\bibitem{Hirsch2011}
Michael Hirsch, Christian~J. Schuler, Stefan Harmeling, and Bernhard Scholkopf.
\newblock {Fast removal of non-uniform camera shake}.
\newblock In {\em 2011 Int. Conf. Comput. Vis.}, pages 463--470. IEEE, 2011.

\bibitem{Ji17}
Na~Ji.
\newblock Adaptive optical fluorescence microscopy.
\newblock {\em Nat Methods}, 14:374--380, 2017.

\bibitem{Kieu2016}
Van~Cuong Kieu, Florence Cloppet, and Nicole Vincent.
\newblock {Local blur correction for document images}.
\newblock In {\em Proc. - Int. Conf. Pattern Recognit.}, number~1, pages
  4059--4064. IEEE, 2016.

\bibitem{Kim2016}
Jinok Kim, Jongsuk Oh, and Rae~Hong Park.
\newblock {Removing non-uniform camera shake using blind motion deblurring}.
\newblock In {\em 2016 IEEE Int. Conf. Consum. Electron. ICCE 2016}, number~2,
  pages 351--352. IEEE, 2016.

\bibitem{cifar10}
Alex Krizhevsky.
\newblock Learning multiple layers of features from tiny images.

\bibitem{LiXu2012}
{Li Xu} and {Jiaya Jia}.
\newblock {Depth-aware motion deblurring}.
\newblock In {\em 2012 IEEE Int. Conf. Comput. Photogr.}, volume~1, pages 1--8.
  IEEE, 2012.

\bibitem{LiuYanZen19}
J.~{Liu}, M.~{Yan}, and T.~{Zeng}.
\newblock Surface-aware blind image deblurring.
\newblock {\em IEEE Transactions on Pattern Analysis and Machine Intelligence},
  pages 1--15, 2019.

\bibitem{Nagy1998}
James~G. Nagy and Dianne~P. O'Leary.
\newblock {Restoring images degraded by spatially variant blur}.
\newblock {\em SIAM J. Sci. Comput.}, 19(4):1063--1082, 1998.

\bibitem{Pan2019}
Liyuan Pan, Yuchao Dai, and Miaomiao Liu.
\newblock {Single Image Deblurring and Camera Motion Estimation With Depth
  Map}.
\newblock In {\em 2019 IEEE Winter Conf. Appl. Comput. Vis.}, pages 2116--2125.
  IEEE, 2019.

\bibitem{Pei2019}
Yanting Pei, Yaping Huang, Qi~Zou, Xingyuan Zhang, and Song Wang.
\newblock {Effects of Image Degradation and Degradation Removal to CNN-based
  Image Classification}.
\newblock {\em IEEE Trans. Pattern Anal. Mach. Intell.}, 14(8):1--1, 2019.

\bibitem{Pozzi20}
Paolo Pozzi, Carlas Smith, Elizabeth Carroll, Dean Wilding, Oleg Soloviev,
  Martin Booth, Gleb Vdovin, and Michel Verhaegen.
\newblock Anisoplanatic adaptive optics in parallelized laser scanning
  microscopy.
\newblock {\em Opt. Express}, 28(10):14222--14236, 2020.

\bibitem{Thiebaut2016}
Éric Thiébaut, Loïc Denis, Ferréol Soulez, and Rahul Mourya.
\newblock {Spatially variant PSF modeling and image deblurring}.
\newblock In Enrico Marchetti, Laird~M. Close, and Jean-Pierre Véran, editors,
  {\em Adaptive Optics Systems V}, volume 9909, pages 2211 -- 2220.
  International Society for Optics and Photonics, SPIE, 2016.

\bibitem{VA}
R.~T. Rockafellar and R.~J. Wets.
\newblock {\em Variational Analysis}.
\newblock Grundlehren Math. Wiss. Springer-Verlag, Berlin, 1998.

\bibitem{Roc70}
R.~Tyrrell Rockafellar.
\newblock {\em Convex {A}nalysis}.
\newblock Princeton Mathematical Series, No. 28. Princeton University Press,
  Princeton, N.J., 1970.

\bibitem{Rod99}
F.~Roddier.
\newblock {\em Adaptive Optics in Astronomy}.
\newblock Cambridge University Press, 1999.

\bibitem{SonHlaBoy14}
M.~Sonka, V.~Hlavac, and R.~Boyle.
\newblock {\em Image {P}rocessing, {A}nalysis, and {M}achine {V}ision}.
\newblock Cengage Learning. Springer, Boston, MA, 2014.

\bibitem{SroKamLu16}
F.~Sroubek, J.~Kamenicky, and Y.~M. Lu.
\newblock Decomposition of space-variant blur in image deconvolution.
\newblock {\em IEEE Signal Process. Lett.}, 23:346--350, 2016.

\bibitem{VorCar01}
Mikhail~A. Vorontsov and Gary~W. Carhart.
\newblock Anisoplanatic imaging through turbulent media: image recovery by
  local information fusion from a set of short-exposure images.
\newblock {\em J. Opt. Soc. Am. A}, 18(6):1312--1324, 2001.

\bibitem{WilSolPozVdoVer17}
Dean Wilding, Oleg Soloviev, Paolo Pozzi, Gleb Vdovin, and Michel Verhaegen.
\newblock Blind multi-frame deconvolution by tangential iterative projections
  ({TIP}).
\newblock {\em Opt. Express}, 25(26):32305--32322, 2017.

\bibitem{Xu2016}
Yuquan Xu, Xiyuan Hu, and Silong Peng.
\newblock {Sharp image estimation from a depth-involved motion-blurred image}.
\newblock {\em Neurocomputing}, 171:1185--1192, 2016.

\bibitem{Yue2015}
Tao Yue, Jinli Suo, and Qionghai Dai.
\newblock {Efficient 3D kernel estimation for non-uniform camera shake removal
  using perpendicular camera system}.
\newblock {\em IEEE Comput. Soc. Conf. Comput. Vis. Pattern Recognit. Work.},
  pages 10--15, 2015.

\end{thebibliography}


\end{document}